\documentclass{article}

\usepackage[nonatbib,preprint]{neurips_2020}

\usepackage[utf8]{inputenc} 
\usepackage[T1]{fontenc}    
\usepackage{hyperref}       
\usepackage{url}            
\usepackage{booktabs}       
\usepackage{amsfonts}       
\usepackage{nicefrac}       
\usepackage{microtype}      

\usepackage[english]{babel}
\usepackage{blindtext}

\usepackage[inline]{enumitem}
\usepackage{amsmath,amssymb,latexsym,graphicx}
\usepackage{url}
\usepackage{algorithm}
\usepackage{amssymb}
\usepackage{mathtools}
\usepackage{graphicx}
\usepackage[algo2e,ruled,vlined,linesnumbered,noend]{algorithm2e}
\usepackage{subcaption}
\usepackage{adjustbox}
\usepackage[usenames,dvipsnames]{color}
\usepackage{soul}
\usepackage{multicol}

\usepackage{subcaption}
\usepackage{hyperref}
\usepackage{adjustbox}
\usepackage{makecell}

\graphicspath{ {img/} }
\newcommand*{\QEDW}{\hfill\ensuremath{\square}}%
\newcommand{\norm}[1]{\left\lVert#1\right\rVert}

\newcommand{\Tfrank}{T}
\newcommand{\br}[1]{\left\{#1\right\}}

\DeclareMathOperator*{\argmax}{arg\,max}
\DeclareMathOperator*{\argmin}{arg\,min}

\newcommand{\comment}[1]{}
\newcommand{\REAL}{\ensuremath{\mathbb{R}}}

\newcommand{\abs}[1]        {\left| #1\right|}

\newcommand{\ceil}[1]{\left \lceil #1 \right \rceil}

\newcommand{\smi}{\sum_{i=1}^n}

\newcommand{\err}{err}
\renewcommand{\hl}{\comment}
\newcommand{\hll}{}
\newenvironment{proof}{\noindent\normalfont {\bf Proof}.\ }{\QEDW \par\vskip 4mm\par}
 \newtheorem{theorem}{Theorem}[section]
 \newtheorem{corollary}[theorem]{Corollary}
 \newtheorem{lemma}[theorem]{Lemma}
 
\newtheorem{definition}[theorem]{Definition}

 \newtheorem{claim}[theorem]{Claim}
 
 \newtheorem{observation}[theorem]{Observation}

\newcommand{\eps}{\varepsilon}

\newcommand{\co}{8}
\DeclarePairedDelimiter\floor{\lfloor}{\rfloor}

\newcommand{\var}{\sigma}
\newcommand{\smip}{\sum_{p\in P}}
\newcommand{\smis}{\sum_{p\in S}}
\newcommand{\size}{\frac{4}{\eps}}

\newcommand{\myk}{\floor{3.5\log{\left(\frac{1}{\delta}\right)}} +1}

\newcommand{\myn}{\frac{4k}{\eps}}
\newcommand{\weaklogdeltacoreset}{\textsc{Prob-Coreset}}

\newif\ifproofs
\proofsfalse

\title{Faster PAC Learning and Smaller Coresets via Smoothed Analysis}

\author{%
  Alaa Maalouf\\
  \texttt{alaamalouf12@gmail.com} \\
\And
  Ibrahim Jubran\\
  \texttt{ibrahim.jub@gmail.com} \\
\AND
  Murad Tukan\\
  \texttt{muradtuk@gmail.com } \\
\And
  Dan Feldman\\
  \texttt{dannyf.post@gmail.com } \\
\AND
\normalfont{The Robotics and Big Data Lab,}\\Department of Computer Science,\\University of Haifa,\\Haifa, Israel
}

\begin{document}

\maketitle

\begin{abstract}
PAC-learning usually aims to compute a small subset ($\eps$-sample/net) from $n$ items, that provably approximates a given loss function for every query (model, classifier, hypothesis) from a given set of queries, up to an additive error $\eps\in(0,1)$. Coresets generalize this idea to support multiplicative error $1\pm\eps$.

Inspired by smoothed analysis, we suggest a natural generalization: approximate the \emph{average} (instead of the worst-case) error over the queries, in the hope of getting smaller subsets. The dependency between errors of different queries implies that we may no longer apply the Chernoff-Hoeffding inequality for a fixed query, and then use the VC-dimension or union bound.

This paper provides deterministic and randomized algorithms for computing such coresets and $\eps$-samples of size independent of $n$, for any finite set of queries and loss function. Example applications include new and improved coreset constructions for e.g.  streaming vector summarization [ICML'17] and $k$-PCA [NIPS'16]. Experimental results with open source code are provided.
\end{abstract}

\section{Introduction}\label{sec:intro}
In this paper we assume that the input is a set $P$ of items, called points. Usually $P$ is simply a finite set of $n$ points in $\REAL^d$ or other metric space. In the context of PAC-learning~\cite{valiant1984theory}, or Empirical Risk Minimization~\cite{vapnik1992principles} it represents the training set. In supervised learning every point in $P$ may also include its label or class. We also assume a given function $w:P\to (0,\infty)$ called \emph{weights function} that assigns a ``weight'' $w(p)>0$ for every point $p\in P$.  The weights function represents a distribution of importance over the input points, where the natural choice is uniform distribution, i.e.,  $w(p)=1/|P|$ for every $p\in P$. We are also given a (possibly infinite) set $X$ that is the set of \emph{queries}~\cite{feldman2011unified} which represents candidate models or hypothesis, e.g. neural networks~\cite{nielsen2015neural}, SVMs~\cite{steinwart2008support} or a set of vectors in $\REAL^d$ with tuning parameters as in linear/ridge/lasso regression~\cite{tibshirani1996regression, hastie2009elements,hoerl1970ridge}.

In machine learning and PAC-learning in particular, we often seek to compute the query that best describes our input data $P$ for either prediction, classification, or clustering tasks. To this end, we define a \emph{loss function} $f:P\times X\to \REAL$ that assigns a fitting cost $f(p,x)$ to every point $p\in P$ with respect to a query $x\in X$. For example, it may be a kernel function~\cite{bergman1970kernel}, a convex function~\cite{eggleston1966convexity}, or an inner product. The tuple $(P,w,X,f)$ is called a \emph{query space} and represents the input to our problem. In this paper we wish to approximate the weighted sum of losses $f_w(P,x)=\sum_{p\in P}w(p)f(p,x)$.
\hll{Other papers considers e.g. covering problems where the maximum error $\max_{p\in P}w(p)f(p,x)$, or the $\ell_z$ loss which is proportional to $\sum_{p\in P}w(p)(f(p,x))^z$ are considered. Such overall loss functions may still be represented as query spaces, at least approximately, by e.g. replacing $f(p,x)$ with $f^z(p,x)$.}

\begin{figure}[t]
  \centering
  \includegraphics[width=\textwidth]{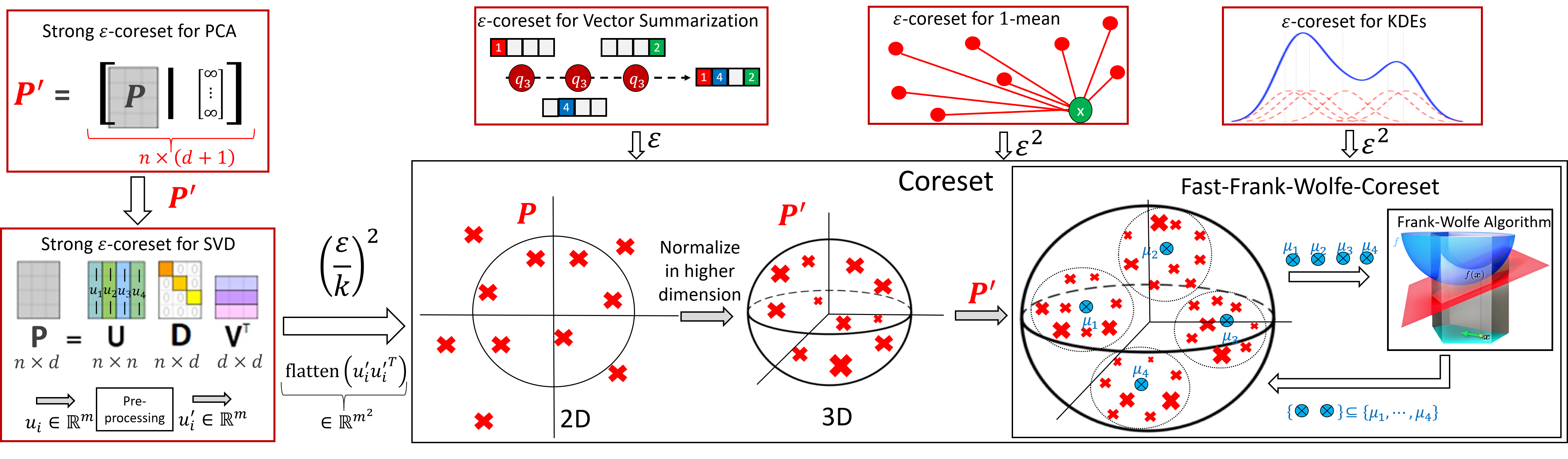}
  \caption{Illustration of Algorithm~\ref{alg:strong}, its normalization of the input, its main applications (red boxes) and their plugged parameters. Algorithm~\ref{alg:strong} utilizes and boosts the run-time of the Frank-Wolfe algorithm for those applications; see Section~\ref{sec:our_contrib}. Some images taken from~\cite{jaggi2013revisiting}.}\label{allthepaper}
\end{figure}


\textbf{$\eps$-Sample. } Suppose that we wish to approximate the mean $f(P,x)=\frac{1}{n}\sum_{p\in P}f(p,x)$ for a specific $x$ in sub-linear time. Picking a random point $p$ uniformly at random from $P$ would give this result in expectation as $E[f(p,x)]=\sum_{p\in P}f(p,x)/n=f(P,x)$. By Hoeffding inequality, the mean $f(S,x)=\frac{1}{|S|}\sum_{p\in S}f(p,x)$ of a uniform sample $S\subseteq P$ would approximate this mean $f(P,x)$ with high probability. More precisely, for a given $\eps\in(0,1)$ the probability of getting an approximation error $\err_f(x)=|f(P,x)-f(S,x)|\leq \eps$ is constant if the size of the sample is $|S|\in  O(M/\eps^2)$ where $M=\max_{p\in P}|f(p,x)|$ is the maximum absolute value of $f$.\hll{ For example, the fraction of voters to a specific candidate is usually predicted via a uniform sample from a population. Such a set $S$ of size $|S|=2$ that yields no error, i.e., $\err(x)=0$, can be computed deterministically (but in linear time) by observing that the mean $f(P,x)$ is always a convex combination (weighted mean) of the pair of points $p_{\min}$ and $p_{\max}$ that minimizes and maximizes $f(p,x)$, respectively.}

Generally, we are interested in such data summarization $S$ of $P$ that approximates \emph{every} query $x\in X$. An \emph{$\eps$-sample} is a pair $(S,u)$ where $S$ is a \emph{subset} of $P$ (unlike e.g. sketches~\cite{phillips2016coresets}), and $u:S\to [0,\infty)$ is its weights function such that the original weighted loss $f_w(P,x)$ is approximated by the weighted loss $f_u(S,x)=\sum_{p\in S}u(p)f(p,x)$ of the (hopefully small) weighted subset $S$~\cite{har2011geometric}, i.e.,
\begin{equation}\label{xx}
\forall x\in X:\left|f_w(P,x)-f_u(S,x)\right|\leq \eps.
\end{equation}
We usually assume that the input is normalized in the sense that $w$ is a distribution, and $f:P\times X\to [-1,1]$.
\hl{The absolute value $|\cdot|$ in {\eqref{xx}} may be replaced by a more general distance function $|\cdot|_\nu$ as explained in {\cite{li2001improved}}. For example, for $\nu=1/4$ yields an $\eps$-net {\cite{braverman2016new}} which is a weaker version of $\eps$-sample where $\forall x\in X:\abs{\sum_{p\in P}w(p)f(p,x)}\geq \eps \rightarrow \abs{\sum_{p\in S}u(p)f(p,x) }>0.$}
By defining the vectors $f_w(P,X)=(\sum_{p\in P}w(p)f(p,x))_{x\in X}$ and $f_u(S,X)=(\sum_{p\in S}u(p)f(p,x))_{x\in X}$, we can define the error for a single $x$ by
$err(x)=\abs{f_w(P,x)-f_u(S,x)}$, and then the error vector for the coreset $\err(X)=(\err(x))_{x\in X}$. We can rewrite~\eqref{xx} by
\begin{equation}\label{err}
\norm{\err(X)}_\infty =\norm{f_w(P,X)-f_u(S,X)}_\infty\leq \eps.
\end{equation}

\paragraph{PAC/DAC learning.} Probably Approximately Correct (PAC) randomized constructions generalizes Hoeffding inequality above from a single to multiple (usually infinite) queries and returns an $\eps$-sample for a given query space $(P,w,X,f)$ and $\delta\in(0,1)$, with probability at least $1-\delta$. Here, $\delta$ corresponds to the ``probably" part, while ``approximately correct" corresponds to $\eps$ in~\eqref{err}; see ~\cite{vapnik2013nature,braverman2016new,langberg2010universal}.
Deterministic Approximately Correct (DAC) versions of PAC-learning suggest deterministic construction of $\eps$-samples, i.e., the probability of failure of the construction is $\delta=0$.

As common in machine learning and computer science in general, the main advantage of deterministic constructions is smaller bounds \hll{(in this case, on the size of the resulting $\eps$-sample)} that cannot be obtained via random sampling \hll{e.g. due to the lower bounds that are related to the Coupon Collector Problem {\cite{flajolet1992birthday}}}. Their disadvantage is usually the slower construction time that may be unavoidable.\hll{, e.g. deterministic version of the Johnson Lindestrauss lemma {\cite{kane2010derandomized}} that takes time {$O(n^2)$} compared to linear or even sub-linear time via uniform sampling {\cite{achlioptas2003database}}.} 
\hl{Deterministic constructions are known even when the VC-dimension is not bounded.}The Caratheodory theorem~\cite{caratheodory1907variabilitatsbereich,cook1972caratheodory} suggests a deterministic algorithm that always returns an (exact) $0$-sample $(S,u)$ of size $|S|\leq |X|+1$, i.e., such that $f_w(P,X)=f_u(S,X)$; see~\cite{jubran2019introduction}.
\hll{This generalizes the deterministic example above for $|X|=1$.}
A general framework to compute deterministic $\eps$-samples in time linear in $n$ but exponential in the VC-dimension was suggested by~\cite{matousek1995approximations}. It is an open problem to compute an $\eps$-sample deterministically in time that is polynomial in the VC-dimension (as in the case of PAC-learning) even for the family of hyperplanes. See recent solutions for special cases in~\cite{batson2012twice}.

\paragraph{Sup-sampling. } As explained above, Hoeffding inequality implies an approximation of $f_w(P,x)$ by $f_u(S,x)$ where $u(p)=1/|S|$ and $S$ is a random sample according to $w$ whose size depends on $M(f)=\max_{p\in P} |f(p,x)|$. To reduce the sample size we may thus define $g(p,x)=\frac{f(p,x)}{|f(p,x)|}\in \br{-1,1}$, and $s(p)=\frac{w(p)|f(p,x)|}{\sum_{q\in P}w(q)|f(p,x)|}$. Now, $M(g)=\max_{p\in P}|g(p,x)|=1$. Now, by Hoeffding inequality, the error of approximating $g_s(P,x)$ via non-uniform random sample of size $1/\eps^2$, drawn from $s$, is $\eps$. Define $T = \sum_{q\in P}w(q)|f(q,x)|$. Since $f_w(P,x)= T \cdot g_s(P,x)$, approximating $g_s(P,x)$ up to $\eps$ error yields an error of $\eps T$ for $f_w(p,x)$. Therefore, the size is reduced from $M(f)/\eps^2$ to $T^2/\eps^2$ when $T^2 \leq M(f)$. Here, we sample $|S|= O(T^2/\eps^2)$ points from $P$ according to the distribution $s$, and re-weight the sampled points by $u'(p)=\frac{T}{|S||f(p,x)|}$.

Unlike traditional PAC-learning, the sample now is non-uniform, and is proportional to $s(p)$, rather than $w$, as implied by Hoeffding inequality for non-uniform distributions.
For sets of queries we generalize the definition for every $p\in P$ to $s(p)=\sup_{x\in X}\frac{w(p)|f(p,x)|}{\sum_{q\in P}w(q)|f(p,x)|}$ \hl{??shouldnt be absolute value since f can be negative ??} as in~\cite{braverman2016new}, which is especially useful for coresets below.




\begin{figure}[t]
\centering
  \includegraphics[scale=0.7]{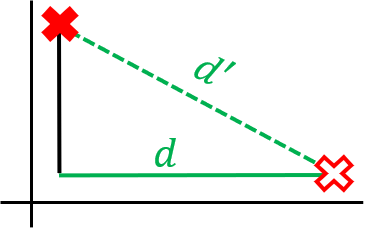}
  \caption{Two points $\left(f_w(p,x_1),f_w(p,x_2)\right)$ (solid red X) and $\left(f_u(p,x_1),f_w(p,x_2)\right)$ (empty red X), where $p$ is an input point, $X = \br{x_1,x_2}$ is the set of possible queries, and $w,u$ are the input and coreset weights respectively. The original sensitivity based coreset insures that $d \leq \varepsilon$, while the new framework in this paper insures that $d'/|X| = d'/2 \leq \varepsilon$. Since $d'/2 \leq d$, our constraint is weaker and thus allows the construction of smaller coresets.}
  \label{fig:averageQueryErr}
\end{figure}

\textbf{Coreset. }Coreset for a given query space $(P,w,X,g)$, in this and many other papers, is a pair $(C,u)$ that is similar to $\eps$-sample in the sense that $C\subseteq P$ and $u:C\to [0,\infty)$ is a weights function. However, the additive error $\eps$ is now replaced with a multiplicative error $1\pm\eps$, i.e., for every $x\in X$, $$|g_w(P,X)-g_u(C,X)|\leq \eps g_w(P,X).$$
Dividing by $g_w(P,X)$ and assuming $g_w(P,X)>0$, yields
\begin{equation}\label{core}
\forall x\in X:  \abs{1-\frac{g_u(C,X)}{g_w(P,x)}}\leq \eps.
\end{equation}

Coresets are especially useful for learning big data since an off-line and possibly inefficient coreset construction for ``small data" implies constructions that maintains coreset for streaming, dynamic (including deletions) and distributed data in parallel. This is via a simple and easy to implement framework that is sometimes called merge-reduce trees; see~\cite{bentley1980decomposable, har2004coresets}.
The fact that a coreset approximates every query (and not just the optimal one for some criterion) implies that we may solve hard optimization problems with non-trivial and non-convex constraints by running a possibly inefficient algorithm such as exhaustive search on the coreset, or running existing heuristics numerous times on the small coreset instead of once on the original data. Similarly, parameter tuning or cross validation can be applied on a coreset that is computed once for the original data as explained in~\cite{maalouf2019fast}.

An $\eps$-coreset for a query space $(P,w,X,g)$ is simply an $\eps$-sample for the query space $(P,w,X,f)$, after defining $f(p,x):=w(p)g(p,x)/g_w(P,x)$, as explained e.g. in~\cite{braverman2016new}. By defining the error for a single $x$ by $err'(x)=\abs{1-g_u(C,x)/g_w(P,x)}=\abs{f_w(P,x)-f_u(C,x)}=\err_f(x)$ we obtain an error vector for the coreset $\err'(X)=(\err'(x))_{x\in X}$. We can then rewrite~\eqref{core} as in~\eqref{err}: $$\norm{\err'(X)}_{\infty}\leq \eps$$.

In the case of coresets, the sup $\sup_{x\in X}f(p,x)=w(p)\sup_{x\in X}g(p,x)/g_w(P,x)$ of a point $p\in P$ is called sensitivity~\cite{langberg2010universal}, leverage score (in $\ell_2$ approximations)~\cite{drineas2012fast}, Lewis weights (in $\ell_p$ approximations) or simply importance sampling~\cite{cohen2015lp}. \hll{Many classic supervised and unsupervised problems in machine learning including PCA, $k$-means and their variants have a corresponding small total sensitivity $\sum_{p\in P}s(p)$, i.e.,  independent of $n=|P|$ or depends only logarithmic on $n$.}



\section{Problem Statement: Smoothed Analysis for Data Summarization. } \label{sec:problemState}
Smoothed Analysis~\cite{spielman2009smoothed} was suggested about a decade ago as an alternative to the (sometimes infamous) worst-case analysis of algorithms in theoretical computer science. The idea is to replace the analysis for the worst-case input by the average input (in some sense). Inspired by this idea, a natural variant of~\eqref{err} and its above implications is an $\eps$-sample that approximates well the \emph{average query}. We suggest to define an $(\eps,\norm{\cdot})$-sample as
\begin{equation}\label{above}
\norm{\err(X)}=\norm{f_w(P,X)-f_u(S,X)} \leq \eps,
\end{equation}
which generalizes~\eqref{err} from $\norm{\cdot}_{\infty}$ to any norm, such as the $\ell_z$ norm $\norm{err(X)}_z$.
For example, for the $\ell_2$, MSE or Frobenius norm, as shown in Fig.~\ref{fig:averageQueryErr}, we obtain
\begin{align}\label{norm2deff}
    \sqrt{\sum_{x\in X}\big(f_w(P,x)-f_u(S,x)\big)^2}\leq \eps.
\end{align}

A generalization of the Hoeffding Inequality from 1963 with tight bounds was suggested relatively recently for the $\ell_z$ norm for any $z\geq 2$ and many other norms~\cite{juditsky2008large,tropp2015introduction}. Here we assume a single query ($|X|=1$), a distribution weights function, and a bound on $\sup_{p\in P}\abs{f(p,X)}$ that determines the size of the sample, as in Hoeffding inequality. \hl{Using the union bound, we can thus obtain $\eps$-samples for finite query sets $|X|$ whose size depend logarithmical on $X$.}  \hl{We do not know of any generalization to the case of infinite query set $X$????--- what about a sample of size $1/(\eps^2\delta) $?????, even with bounded VC-dimension, but we expect that similar usage of the VC-dimension can be applied to the existing proofs in {\cite{alon2004probabilistic}} {\cite{li2001improved}} }

A less obvious question, which is the subject of this paper, is how to compute deterministic $\eps$-samples that satisfies~\eqref{above}, for norms other than the infinity norm. While Caratheodory theorem suggests deterministic constructions of $0$-samples (for any error norm) as explained above, our goal is to obtain coreset whose size is smaller or independent of $|X|$.

The next question is how to generalize the idea of sup-sampling, i.e., where the function $f$ is unbounded, for the case of norms other than $\norm{\cdot}_{\infty}$. Our main motivation for doing so is to obtain new and smaller coresets by combining the notion of $\eps$-sample and sup-sampling or sensitivity as explained above for the $\norm{\cdot}_{\infty}$ case. That is, we wish a coreset for a given query space, that would bound the non-$\ell_\infty$ norm error
$$\norm{\left(1-\frac{g_u(C,x)}{g_w(P,x)}\right)_{x\in X}}=
\norm{\err'(X)}\leq \eps.
$$

To summarize, our questions are:
\textbf{How can we smooth the error function and approximate the ``average'' query via: (i) Deterministic $\eps$-samples (for DAC-learning )? (ii) Coresets (via sensitivities/sup sampling for non-infinity norms)?}

\section{Our contribution}
\label{sec:our_contrib}
We answer affirmably these questions by suggesting $\eps$-samples and coresets for the average query. We focus on the case $z=2$, i.e., the Frobenius norm, and finite query set $X$ and hope that this would inspire the research, applications of other norms and general sets. For suggestions in this direction and future work see Section~\ref{sec:FW}.
The main results of this paper are the following constructions of an $(\eps,\norm{\cdot}_2)$-sample $(S,u)$ for any given finite query space $(P,w,X,f)$ as defined in~\eqref{norm2deff}:
\begin{enumerate}[label=(\roman*)]
    \item Deterministic construction that returns a corest of size $|S|\in O(1/\eps^2)$ in time $O\left(\min\br{nd/\eps^2,nd+d\log(n)^2/\eps^4}\right)$; see Theorem~\ref{weak-coreset-theorem} and Corollary~\ref{weak-coreset-cor}.

    \item Randomized construction that returns such a coreset (of size $|S|\in O(1/\eps^2)$) with probability at least $1-\delta$ in sub-linear time $O\left(d\left (\log{(\frac{1}{\delta})}^2 + \frac{\log{(\frac{1}{\delta})}}{\eps^2} \right) \right)$; see Lemma~\ref{prob:onemean}.
\end{enumerate}



\paragraph{Applications. }
\begin{enumerate}[label=(\roman*)]

\item  Vector summarization: maintain the sum of a (possibly infinite) stream of vectors in $\REAL^d$, up to an additive error of $\eps$ multiplied by their variance. This is a generalization of frequent items/directions~\cite{charikar2002finding}.
We propose a \emph{deterministic} algorithm which reduces each subset of $n$ vectors into $O(1/\eps)$ weighted vectors in $O(nd+d\log(n)^2/\eps^2)$ time, improving upon the $nd/\eps$ of~\cite{feldman2017coresets}, for a sufficiently large $n$; see Corollary~\ref{weak-coreset-cor}, and Fig.~\ref{fig:TwitterSizeTime}. We also provide a non-deterministic coreset construction in Lemma~\ref{prob:onemean}. The merge-and-reduce tree can then be used to support streaming, distributed or dynamic data. 

\item By replacing $\eps$ with $\eps^2$ for the vector summarization, we obtain fast construction of an $\eps$-coreset for Kernel Density Estimates (KDE) of Euclidean kernels~\cite{phillips2019near}; see Section~\ref{sec:svdAndPcaExt}.

\item Coreset for $1$-mean which approximates the sum of squared distances over a set of $n$ points to \emph{any} given center (point) in $\REAL^d$. The deterministic construction computes such a weighted \emph{subset} of size $O(1/\eps^2)$ in time $O(\min\br{nd/\eps^2,nd+d\log(n)^2/\eps^4})$. Previous results of~\cite{bachem2018scalable,barger2016k,DBLP:journals/corr/BravermanFL16,feldman2018turning} suggested coresets for such problem. However, those works are either non-deterministic, do not return a subset of the input, or return a set of size linear in $d$. 

\item Coreset for LMS solvers and dimensionality reduction. For example, a deterministic construction that gets a matrix $A\in\REAL^{n\times d}$ and returns a weighted subset of $k^2/\eps^2$ rows, such that their weighted distance to any $k$-dimensional affine or non-affine subspace approximates the distance of the original points to this subspace. In this paper we propose a \emph{deterministic} coreset construction that takes $O(nd^2 + d^2k^4\log(n)^2/\eps^4)$ time, improving upon the state of the art $O(nd^2k^2/\eps^2)$ time of~\cite{feldman2016dimensionality}; see Table~\ref{table:relatedWork}.
Many non-deterministic coresets constructions were suggested for those problems, the construction techniques apply non-uniform sampling~\cite{cohen2015dimensionality,varadarajan2012sensitivity,feldman2015more}, Monte-Carlo sampling~\cite{frieze2004fast}, and leverage score sampling~\cite{yang2017weighted,drineas2012fast,cohen2015uniform,papailiopoulos2014provable,drineas2008relative,cohen2017input,maalouf2019tight}. However, those works are non-deterministic.

\end{enumerate}

\newcommand{\poly}{poly}
\begin{table}[t]
\centering
\caption{\textbf{Known deterministic subset coresets for LMS solvers.} Our result has the fastest running time for sufficiently large $n$ and $d$.}
\begin{tabular}{|c|c|c|c|c|}
\hline
 Error & Size & Time & Citation & Notes\\
\hline
 $\eps$ & $O(k^2/\eps^2)$ & $O(nd^2/\eps^2)$  & \cite{feldman2016dimensionality} & N/A \\
 $\eps$ & $O(d/\eps^2)$  & $\poly(n,d,\eps)$ &   \cite{batson2012twice} & inefficient for large $n$ \\
  $0$   & $O(d^2)$       & $O(nd^2 + \log(n)\poly(d))$  & \cite{maalouf2019fast} & inefficient for large $d$\\
  $\eps$ & $O(k/\eps^2)$  &  $\poly(n,d,k,\eps)$ & \cite{cohen2015optimal} & inefficient for large $n$\\
  $\eps$ &  $O(k^2/\eps^2)$ & $O(nd^2 + \log(n)^2 d^2k^4/\eps^4)$ & $\star$ & N/A\\
\hline
\end{tabular}
\label{table:relatedWork}
\end{table}

\section{Preliminaries}
\paragraph{Notations.} We denote by $[n] = \br{1,\cdots, n}$. For a vector $v\in\REAL^d$, the $0$-norm is denoted by $\norm{v}_0$ and is equal to the number of non-zero entries in $v$. We denote by $e_{(i)}$ the $i$th standard basis vector in $\REAL^n$ and by $\mathbf{0}$ the vector $(0,\cdots,0)^T \in \REAL^n$.
A vector $w\in [0,1]^n$ is called a distribution vector if all its entries are non-negative and sums up to one.
For a matrix $A \in \REAL^{m\times n}$ and $i\in [m], j\in [n]$ we denote by $A_{i,j}$ the $j$th entry of the $i$th row of $A$.
A \textit{weighted set} is a pair $(Q,m)$ where $Q=\br{q_1,\cdots,q_n} \subseteq \REAL^d$ is a set of $n$ \emph{points}, and $m=(m_1,\cdots,m_n)^T \in \REAL^n$ is a \emph{weights vector} that assigns every $q_i\in Q$ a weight $m_i \in \REAL$. A matrix $X \in \REAL^{d'\times d}$ is \emph{orthogonal} if $X^TX = I \in \REAL^{d\times d}$.

\paragraph{Adaptations. }To adapt to the notation of the following sections and the query space $(P,w,X,f)$ to the techniques that we use, we restate~\eqref{above} as follows. Previously, we denote the queries $X=\br{x_1,\cdots,x_d}$, and the input set by $P=\br{p_1,\cdots,p_n}$. Each input point $p_i\in P$ now corresponds to a point $q_i=\big(f(p_i,x_1),\cdots,f(p_i,x_d)\big)\in \REAL^{d}$, i.e., each entry of $q_i$ equals to $f(p_i,x)$ for a different query $x$. Throughout the rest of the paper, for technical reason and simplicity, we might alternate between the weights function notation and a weights vector notation. In such cases,  $w:P\to[0,\infty)$ is replaced by a vector $m\in[0,\infty)^n$ of weights, where $m_i$ is the weight of $q_i$ for $i\in[n]$, and vice versa. In such cases, the $\eps$-sample is represented by a sparse vector $u\in [0,\infty)$ where $S=\br{p_i\in P\mid u_i>0, i\in[n]}$ is the chosen subset of $P$. Hence, $f_w(P,X)=\sum_{p\in P}w(p)\big(f(p,x_1),\cdots,f(p,x_d)\big)=\sum_{i=1}^n m_i q_i$,
and $f_u(S,X)=\sum_{i=1}^n u_i q_i$.



\hl{Similarly, $u$ is a coreset if $\norm{(m_i-u_i)q_i} \leq \eps \sqrt{\sum_i m_i \norm{q_i}^2}$, and the sensitivity of the $i$th input point is $s_i=m_i\norm{q_i}_2$.
Note that the definition can be generalized to any $\ell_z$ norm, where $z=\infty$ yields the existing $\eps$-samples, sensitivities and coresets.}

\paragraph{From $(\eps,\norm{\cdot}_2)$-samples to $\eps$-coresets.} In what follows is the definition of $\eps$-coreset for vector summarization, which is a re-weighting of the input weighted set $(Q,m)$ by a new weights vector $u$, such that the squared norm between the weighted means of $(Q,u)$ and $(Q,m)$ is small. This relates to Section~\ref{sec:problemState}, where an $(\sqrt{\eps},\norm{\cdot}_2)$-sample there is an
$\eps$-coreset for vector summarization here.
\begin{definition} [vector summarization $\eps$-coreset] \label{def:weakCoreset}
Let $(Q,m)$ and $(Q,u)$ be two weighted sets of $n$ points in $\REAL^d$, and $\eps \in [0,1)$. Let $\mu=\sum_{i=1}^{n} \frac{m_i}{\norm{m}_1}  q_i$, $\sigma^2 = \smi \frac{m_i}{ \norm{m}_1}\norm{q_i- \mu}^2$, and $\tilde{\mu} = \sum_{i=1}^{n} \frac{u_i}{ \norm{u}_1} q_i$.
Then $(Q,u)$ is a \emph{vector summarization $\eps$-coreset} for $(Q,m)$ if $ \norm{\tilde{\mu} - \mu}_2^2 \leq \eps\sigma^2$.
\end{definition}

\section{Vector Summarization Coreset Construction}\label{sec:detCoreset}

In what follows we assume that the points of $P$ lie inside the unit ball ($\forall_{p\in P}: \norm{p}\leq 1$). For such an input set, we present a construction of a variant of a vector summarization coreset, where the error is $\eps$ and does not depend on the variance of the input. This construction is based on the Frank-Wolfe Algorithm~\cite{clarkson2010coresets}; see Theorem~\ref{thm1} and Algorithm~\ref{alg:frankwolf}.
We then present a proper coreset construction in Algorithm~\ref{alg:strong} and Theorem~\ref{weak-coreset-theorem} for a general input set $Q$ in $\REAL^d$. This algorithm is based on a reduction to the simpler case of points inside the unit ball; see Fig.~\ref{allthepaper} for illustration. 

\newcommand{\frankwolf}{\textsc{Frank-Wolf}}

\setcounter{AlgoLine}{0}
\begin{algorithm}[t!]
\setcounter{AlgoLine}{0}
\SetAlgoLined
\caption{$\frankwolf(f,K)$; Algorithm~1.1 of~\cite{clarkson2010coresets}\label{alg:frankwolf}}
\SetKwInOut{Input}{Input}
\SetKwInOut{Output}{Output}
\Input{A concave function $f:\REAL^n \to \REAL$, and the number of iterations $K$.}
\Output{A vector $x \in \REAL^n$ that satisfies Theorem~\ref{thm1} }

Pick as $x_{(0)}$ the vertex of $S$ with largest $f$ value.

\For{$k \in \br{0, \cdots, K}$}{
    $i' := \argmax_{i}\nabla f(x_{(k)})_i $

    $\displaystyle\alpha' := \argmax_{\alpha_{[0,1]}} f\left(x_{(k)} + \alpha (e_{(i')} -x_{(k)})\right)$

    $x_{(k+1)} := x_{(k)} + \alpha'(e_{(i')} -x_{(k)} )$

    \Return $x_{(k+1)}$
}
\end{algorithm}

\newcommand{\strongepscoreset}{\textsc{CoreSet}}
\newcommand{\weak}{\textsc{Weak-CoreSet}}

\newcommand{\boost}{\textsc{Fast-FW-CoreSet}}
\setcounter{AlgoLine}{0}
\begin{algorithm}[t!]
\caption{$\strongepscoreset(Q,m,\eps)$\label{alg:strong}}
\SetKwInOut{Input}{Input}
\SetKwInOut{Output}{Output}
\Input{A weigthed set $(Q,m)$ of $n\geq 2$ points in $\REAL^d$ and an error parameter $\eps\in (0,1)$.}
\Output{A weight vector $u\in[0,\infty)^n$ with $O(1/\eps)$ non-zero entries that satisfies Theorem~\ref{weak-coreset-theorem}.}
    $w=\frac{m}{\norm{m}_1}$ \label{wcomp}

    $\mu_w:= \smi w_iq_i$ \label{mucomp}

    $\sigma_w=\sqrt{\smi w_i\norm{q_i - \mu}^2 }$\label{varcomp}

	\For{every $i \in \{1,\cdots,n\}$ \label{LineFor}} {
	    $p_i=\frac{q_i - \mu}{\sigma}$
	
		$\displaystyle{p'_i  := \frac{(p^T_i\mid 1)^T}{\norm{(p^T_i\mid 1)}^2}}$ \label{p'} \quad\tcp{$\norm{p_i'} \leq 1$.}
		
		$\displaystyle{w'_i :=\frac{ w_i\norm{(p^T_i\mid1)}^2}{2}}$ \label{w'}
	}	
	Compute a sparse vector $u^\prime$ with $O(1/\eps^2)$ non-zero entries, such that $\norm{\smi (w'_i - u_i')p'_i}^2\leq \eps$\label{define u'} \\ \tcp{E.g., using Algorithm~\ref{alg:frankwolf} (see Theorem~\ref{thm1}).}
	
	\For{every $i \in \{1,\cdots,n\}$ \label{LineFor2}} {
		$\displaystyle{u_i =\norm{m}_1 \cdot \frac{2 u_i'}{\norm{(p^T_i\mid 1)}^2}}$ \label{u}
	}
	\Return ${u}$
\end{algorithm}

\begin{theorem}[\textbf{Coreset for points inside the unit ball\label{thm1}}]
Let $P=\{p_1,\cdots,p_n\}$ be a set of n points in $\REAL^d$ such that $\norm{p_i}\leq 1$ for every $i\in [n]$. Let $\eps\in(0,1)$ and $w=(w_1,\cdots,w_n)^T\in[0,1]^n$ be a distribution vector. For every $x=(x_1,\cdots,x_n)^T\in \REAL^n$, define ${f(x) =-\norm{\smi (w_i-x_i)p_i}^2}$. Let  $\tilde{u}$ be the output of a call to $\frankwolf(f,\ceil{\frac{8}{\eps}})$; see Algorithm~\ref{alg:frankwolf}. Then
\begin{enumerate*}[label=(\roman*)]
    \item $\tilde{u}$ is a distribution vector with $\norm{\tilde{u}}_0\leq \ceil{\frac{8}{\eps}} $, \label{prop1neded}
    \item $\norm{\smi(w_i-\tilde{u_i})p_i}^2\leq \eps$,\label{prop2neded} and
    \item $\tilde{u_i}$ is computed in $O\left(\frac{nd}{\eps}\right)$ time.\label{prop4neded}
\end{enumerate*}
\end{theorem}

We now show how to obtain a vector summarization $\eps$-coreset of size $O(1/\eps)$ in $O(\frac{nd}{\eps})$ time.
\begin{theorem}\label{weak-coreset-theorem} Let $(Q,m)$ be a weighted set of $n$ points in $\REAL^d$, $\eps\in (0,1)$, and let $u$ be the output of a call to $\strongepscoreset(Q,m,\frac{\eps}{16})$;  see Algorithm~\ref{alg:strong}.
Then, $u\in \REAL^n$ is a vector with $\norm{u}_0\leq \frac{128}{\eps}$ non-zero entries that is computed in $O(\frac{nd}{\eps})$ time, and $(Q,u)$ is a vector summarization $\eps$-coreset for $(Q,m)$.
\end{theorem}

\section{Boosting the running time}\label{boosterboost}
In this section, we present Algorithm~\ref{algboost}, which aims to boost the running time of Algorithm~\ref{alg:frankwolf} from the previous section. This immediately implies a faster construction time of vector sumarization $\varepsilon$-coresets for general input sets; see Corollary~\ref{weak-coreset-cor} and Fig.~\ref{allthepaper} for illustration.
Then, in Lemma~\ref{prob:onemean}, we show how to compute a vector summarization coreset with high probability in a time that is sublinear in the input size $n$.



\setcounter{AlgoLine}{0}
\begin{algorithm}[H]
\SetAlgoLined
\caption{$\boost(P,w,\eps)$\label{algboost}}
\SetKwInOut{Input}{Input}
\SetKwInOut{Output}{Output}
\Input{A weighed set $(P,w)$ of $n\geq 2$ points in $\REAL^d$ and an error parameter $\eps\in (0,1)$.}
\Output{A pair $(C,u)$ that satisfies Theorem~\ref{theorem:fastcoreset}}

$k:= \frac{2\log(n)}{\eps}$

\If{ $|P|\leq k$ \label{stoprule}}{ \label{ifififif}
\Return an vector summarization $\eps$-corset for $(P,w)$ using Theorem~\ref{thm1}. \label{returnInput} }

$\br{P_1,\cdots,P_k}:=$ a partition of $P$ into $k$ disjoint subsets, each contains at most $\ceil{n/k}$ points. \label{compPartition}

\For{every $i\in\br{1,\cdots,k}$}
{
$\displaystyle \mu_i:=\frac{1}{\sum_{q\in P_i}w(q)} \cdot \sum_{p\in P_i}w(p)\cdot p$ \\\tcp{The weighted mean of $P_i$} \label{compMui}
$w'(\mu_i) := \sum_{p\in P_i}w(p)$ \label{compMuiWeight}
}

$(\tilde{\mu},\tilde{u}) :=$ a vector summarization $\left(\frac{\eps}{\log(n)}\right)$-corset for the weighted set  $(\br{\mu_1,\cdots,\mu_k},w')$ using Theorem~\ref{thm1}. \label{compSlowCara} \\

$\displaystyle C:=\bigcup_{\mu_i\in \tilde{\mu}} P_i$ \label{compC} \tcp{$C$ is the union over all subsets $P_i$ whose mean $\mu_i$ was chosen in $\tilde{\mu}$.}

\For {every $\mu_i \in \tilde{\mu}$ and $p\in P_i$\label{forRemainingClusters}}
{
$\displaystyle u(p):=\frac{\tilde{u}(\mu_i)w(p)}{\sum_{q\in P_i}w(q)}$ \label{compW}\\
}
$(C,u):=\boost(C,u, \eps)$

\Return $(C,u)$ \label{recursiveCall}

\end{algorithm}


\begin{theorem}[Faster coreset for points inside the unit ball] \label{theorem:fastcoreset}
Let $P$ be a set of n points in $\REAL^d$ such that $\norm{p}\leq 1$ for every $p\in P$. Let $w:P \to (0,1)$ be a weights function such that $\sum_{p\in P} w(p)=1$ and let $\eps\in(0,1)$. Let $(C,{u})$ be the output of a call to $\boost(P,w,\eps)$; see Algorithm~\ref{algboost}. Then
\begin{enumerate*}[label=(\roman*)]
    \item $\abs{C}\leq \co/\eps$ and $\sum_{p\in C} {u}(p)=1$, \label{smalsubsetcp}
    \item $\norm{\sum_{p\in P}w(p) p -\sum_{p\in C}{u}(p)p}^2\leq 2\eps$, and \label{epsweakcor}
    \item $(C,{u})$ is computed in $O\left(nd + \frac{d\cdot \log{(n)}^2}{\eps^2} \right)$ time.\label{runtimeofcorest}
\end{enumerate*}

\end{theorem}

\begin{corollary}\label{weak-coreset-cor}
Let $(Q,m)$ be a weighted set of $n$ points in $\REAL^d$, and let $\eps\in (0,1)$. Then in $O(nd + d \cdot \frac{\log{(n)}^2}{\eps^2})$ time, we can compute a vector $u=(u_1,\cdots,u_n)^T\in \REAL^n$, such that $u$ has $\norm{u}_0\leq 128/\eps$ non-zero entries and $(Q,u)$ is a vector summarization $(2\eps)$-coreset for $(Q,m)$.
\end{corollary}

\begin{algorithm}[H]
\setcounter{AlgoLine}{0}
\SetAlgoLined
\caption{$\weaklogdeltacoreset(Q,\eps,\delta)$\label{alg:smallProb}}
\SetKwInOut{Input}{Input}
\SetKwInOut{Output}{Output}
\Input{A set $Q$ of $n\geq 2$ points in $\REAL^d$,
$\eps\in (0,1)$, and $\delta \in (0,1)$.}
\Output{A subset $S\subseteq P$ that satisfies Lemma~\ref{weak-probabilistic-coreset-theorem}.}
$k:= \myk$. \label{line:defk}

$S:=$ an i.i.d sample of size $\myn$.

$\br{S_1, \cdots, S_{k}} :=$ a partition of $S$ into $k$ disjoint subsets, each contains $\size$ points .\label{line:sapledsets}

$\overline{s}_i:=$ the mean of the $i$'th subset $S_i$ for $i \in [k]$. \label{line:si}

${i^{*}:=\argmin_{j \in [k]}\sum_{i=1}^{k}\norm{\overline{s}_i -\overline{s}_j }_2}$. \label{closest_mean_33}

\Return $S_{i^{*}}$
\end{algorithm}

\begin{lemma}\label{weak-probabilistic-coreset-theorem}\label{prob:onemean}
Let $Q$ be a set of $n$ points in $\REAL^d$, $\mu = \frac{1}{n}\smip q$, and $\var^2=\frac{1}{n}\smip \norm{q - \mu}^2$. Let $\eps\in (0,1)$, $\delta \in (0,0.9]$, and let $S\subseteq \REAL^{d}$ be the output of a call to $\weaklogdeltacoreset(Q,\eps,\delta)$; see Algorithm~\ref{alg:smallProb}. Then (i) $S\subseteq Q$ and $|S| = \frac{4}{\eps}$, (ii) with probability at least $1-3\delta$ we have
$  \norm{  \frac{1}{\abs{S}} \sum_{p\in S} p - \mu }^2 \leq  33 \cdot \eps\var^2$, and (iii) $S$ is computed in $O\left(d\log{(\frac{1}{\delta})}^2 + \frac{d\log{(\frac{1}{\delta})}}{\eps}  \right)$ time.
\end{lemma}


\section{Applications}\label{sec:svdAndPcaExt}


\textbf{Coreset for $1$-mean.}
A $1$-mean $\eps$-coreset for $(Q,m)$ is a weighted set $(Q,u)$ such that for \emph{every} $x\in \REAL^d$, the sum of squared distances from $x$ to either $(Q,m)$ or $(Q,u)$, is approximately the same. The following theorem computes such a coreset.
\begin{theorem}\label{strong-coreset-theorem} Let $(Q,m)$ be a weighted set of $n$ points in $\REAL^d$, $\eps\in (0,1)$. Then in $O(\min \br{nd + d\cdot \frac{\log(n)^2}{\eps^4},\frac{nd}{\eps^2}})$ time we can compute a vector $u=(u_1,\cdots,u_n)^T\in \REAL^n$, where $\norm{u}_0\leq \frac{128}{\eps^2}$, and for every $x\in \REAL^d$: $$ \left|\smi (m_i-u_i)\norm{q_i-x}^2 \right| \leq \eps \smi m_i\norm{q_i-x}^2.$$ 
\end{theorem}

\textbf{Coreset for KDE.} Given a kernel defined by the kernel map $\phi$ and two sets of points $Q$ and $Q'$, the maximal difference $\norm{KDE_Q - KDE_{Q'}}_\infty$ between the kernel costs of $Q$ and $Q'$ is upper bounded by $\norm{\mu_{\hat{Q}} - \mu_{\hat{Q'}}}_2$, where $\mu_{\hat{Q}}$ and $\mu_{\hat{Q'}}$ are the means of $\hat{Q}=\br{\phi(q) \mid q\in Q}$ and $\hat{Q}'=\br{\phi(q) \mid q\in Q'}$ respectively~\cite{desai2016improved}. Given $\hat{Q}$, we can compute a vector summarization $\eps^2$-coreset $\hat{Q}'$, which satisfies that $\norm{\mu_{\hat{Q}} - \mu_{\hat{Q'}}}_2^2 \leq \varepsilon^2$. By the above argument, this is also an $\varepsilon$-KDE coreset.

\newcommand{\pcaalg}{\textsc{PCA-CoreSet}}
\newcommand{\svdalg}{\textsc{SVD-CoreSet}}
\newcommand{\dimReductionAlg}{\textsc{DIM-CoreSet}}

\setcounter{AlgoLine}{0}
\begin{algorithm}[h]
\SetAlgoLined
\caption{$\dimReductionAlg(A,k,\eps)$\label{svdccore}}
\SetKwInOut{Input}{Input}
\SetKwInOut{Output}{Output}
\Input{A matrix $A\in \REAL^{n\times d}$, an integer $k \in [d]$, and an error parameter $\eps\in (0,1)$.}
\Output{A diagonal matrix $W\in \REAL^{n\times n}$ that satisfies Corollary~\ref{dim-cor}}

$\displaystyle{r := 1+ \max_{i\in [n]}\frac{4\norm{a_i}^2}{\eps^4}}$  \label{lineneeded} \tcp{where $a_i$ is the $i$th row of $A$ for every $i\in [n]$}

$U,\Sigma,V$ := the full SVD of the concatenated matrix $[A \mid (r,\cdots,r)^T] \in \REAL^{n\times (d+1)}$

$v_i := \left( U_{i,1}, \cdots ,U_{i,k},\frac{U_{i,k+1:d}\Sigma_{k+1:d,k+1:d}} {\norm{\Sigma_{k+1:d,k+1:d}}_F}\right)$ for every $i\in [n]$

$\tilde{v_i}$:= the row stacking of $v_iv_i^T\in \REAL^{d\times d}$ for every $i \in [n]$

$(\br{\tilde{v_1},\cdots,\tilde{v_n}}, u)$ :=  a vector summarization $(\frac{\eps}{5k})^2$-coreset for
    $(\br{\tilde{v_1},\cdots,\tilde{v_n}},(1,\cdots,1))$.

$W:=$ a diagonal matrix in $\REAL^{n\times n}$, where $W_{i,i} = \sqrt{u_i}$ for every $i\in [n]$.

\Return $W$

\end{algorithm}

\textbf{Coreset for dimensionality reduction and LMS solvers.} Given a matrix $A\in \REAL^{n\times d}$, an $\eps$-coreset for the $k$-SVD ($k$-PCA) problem of $A$ is a small scaled subset of its rows that approximates their sum of squared distances to every non-affine (affine) $k$-dimensional subspace of $\REAL^d$, up to a multiplicative factor of $1\pm\eps$; see Corollary~\ref{dim-cor}. Coreset for LMS solvers is the special case of $k=d-1$.

\begin{corollary}[Coreset for dimensionality reduction]\label{dim-cor}
Let $A\in \REAL^{n \times d}$ be a matrix, $\eps\in(0,\frac{1}{2})$ be an error parameter, $k\in [d]$ be an integer, and $W$ be the output of a call to $\dimReductionAlg(A,k,\eps)$. Then: (i) $W$ is a diagonal matrix with $O\left(\frac{k^2}{\eps^2}\right)$ non-zero entries, (ii)  $W$ is computed in $O\left(\min\br{nd^2 + \frac{d^2\log(n)^2k^4}{\eps^4}, \frac{nd^2k^2}{\eps^2}}\right)$ time, and (iii) there is a sufficiently large constant $c$, such that for every $\ell \in \REAL^d$ and an orthogonal $X \in \REAL^{d\times(d-k)}$ we have $$\abs{1- \frac{\norm{W(A-\ell)X}_F^2}{\norm{(A-\ell)X}_F^2}} \leq c\eps.$$ Here, $A-\ell$ is the subtraction of $\ell$ from every row of $A$.
\end{corollary}

\section{Experimental Results}\label{sec:ER}
In this section we apply different coreset construction algorithms presented in this paper to a variety of applications, in order either to boost their running time, or to reduce their memory storage.

\paragraph{Software / Hardware.}
The algorithms were implemented in Python $3.6$~\cite{10.5555/1593511} using ``Numpy''~\cite{oliphant2006guide}. Tests were conducted on a PC with Intel i9-7960X CPU @2.80GHz x 32 and 128Gb RAM.

\newcommand{\uniform}{\texttt{Uniform}}
\newcommand{\sensonemean}{\texttt{Sensitivity-sum}}
\newcommand{\icmlfeldman}{\texttt{ICML17}}
\newcommand{\nipsfeldman}{\texttt{NIPS16}}
\newcommand{\slowonemean}{\texttt{Our-slow-sum}}
\newcommand{\fastonemean}{\texttt{Our-fast-sum}}
\newcommand{\slowsvd}{\texttt{Our-slow-svd}}
\newcommand{\fastsvd}{\texttt{Our-fast-svd}}
\newcommand{\senssvd}{\texttt{Sensitivity-svd}}
\newcommand{\ourrand}{\texttt{Our-rand-sum}}

\paragraph{Algorithms.}
We compared the following algorithms:
\begin{enumerate}[label=(\roman*)]
    \item \uniform: Uniform sampling.
    \item \sensonemean: vector summarization`` sensitivity'' sampling~\cite{tremblay2019determinantal}.
    \item \icmlfeldman: Algorithm~$2$ in~\cite{feldman2017coresets}.
    \item \ourrand: Our coreset construction from Lemma~\ref{prob:onemean}.
    \item \slowonemean: Our coreset construction from Corollary~\ref{weak-coreset-theorem}.
    \item \fastonemean: Our coreset construction from Corollary~\ref{weak-coreset-cor}.
    \item \senssvd: Sensitivty for $k$-SVD~\cite{varadarajan2012sensitivity}.
    \item \nipsfeldman:  Algorithm~$2$ in~\cite{feldman2016dimensionality}.
    \item \slowsvd{} and \textbf{(x) }\fastsvd: Corollary~\ref{dim-cor} offers a coreset construction for SVD using Algorithm~\ref{svdccore}, which utilizes Algorithm~\ref{alg:strong}. However, Algorithm~\ref{alg:strong} either utilizes Algorithm~\ref{alg:frankwolf} (see Theorem~\ref{weak-coreset-theorem}) or Algorithm~\ref{algboost} (see Corollary~\ref{weak-coreset-cor}). \slowsvd{} applies the former option while \fastsvd{} uses the latter option. See section~\ref{sec:algcomp} at the Appendix for full details about the competing algorithms.
\end{enumerate}


\paragraph{Datasets.} We used the following datasets from the UCI machine learning library~\cite{Dua:2019}:
\begin{enumerate}[label=(\roman*)]
\item \textbf{New York City Taxi Data~\cite{illinoisdatabankIDB-9610843,article}. }The data covers the taxi operations at New York city. We used the data describing $n=14.7M$ trip fares at the year of 2013. We used the $d = 6$ numerical features, which are all real numbers.

\item\textbf{US Census Data (1990) Data Set~\cite{dataset:sales}. }The dataset contains $n=2.4M$ entries. We used the entire $d=68$ real-valued attributes of the dataset.

\item \textbf{Buzz in social media Data Set~\cite{kawala2013predictions}. }It contains $n = 0.5M$ examples of buzz events from two different social networks: Twitter, and Tom's Hardware. We used the entire $d=77$ real-valued attributes.
\end{enumerate}
\paragraph{Experiments.}
\begin{enumerate}[label=(\roman*)]
\item \textbf{vector summarization: }The goal is to approximate the mean of the input using a weighted subset. The approximation error is $\norm{\mu - \mu_s}^2$, where $\mu$ is the mean of the full data and $\mu_s$ is the mean of the subset computed by each algorithm; see Fig.~\ref{fig:triperr}--\ref{fig:UStime}.

\item \textbf{$k$-SVD: }We compute the optimal $k$-dimensional non-affine subspaces $S^*$ and $S'$ either using the full data or using the subset at hand, respectively. The approximation error is defined as the ratio $\abs{(c^*-c')/c^*}$, where $c'$ and $c^*$ are the sum of squared distances between the rows of the full input matrix to $S'$ and $S^*$ respectively.
\end{enumerate}


\newcommand\s{0.24}
\begin{figure*}[t!]
\centering
    \begin{subfigure}[t]{\s\textwidth}
		\centering
		\includegraphics[width = \textwidth]{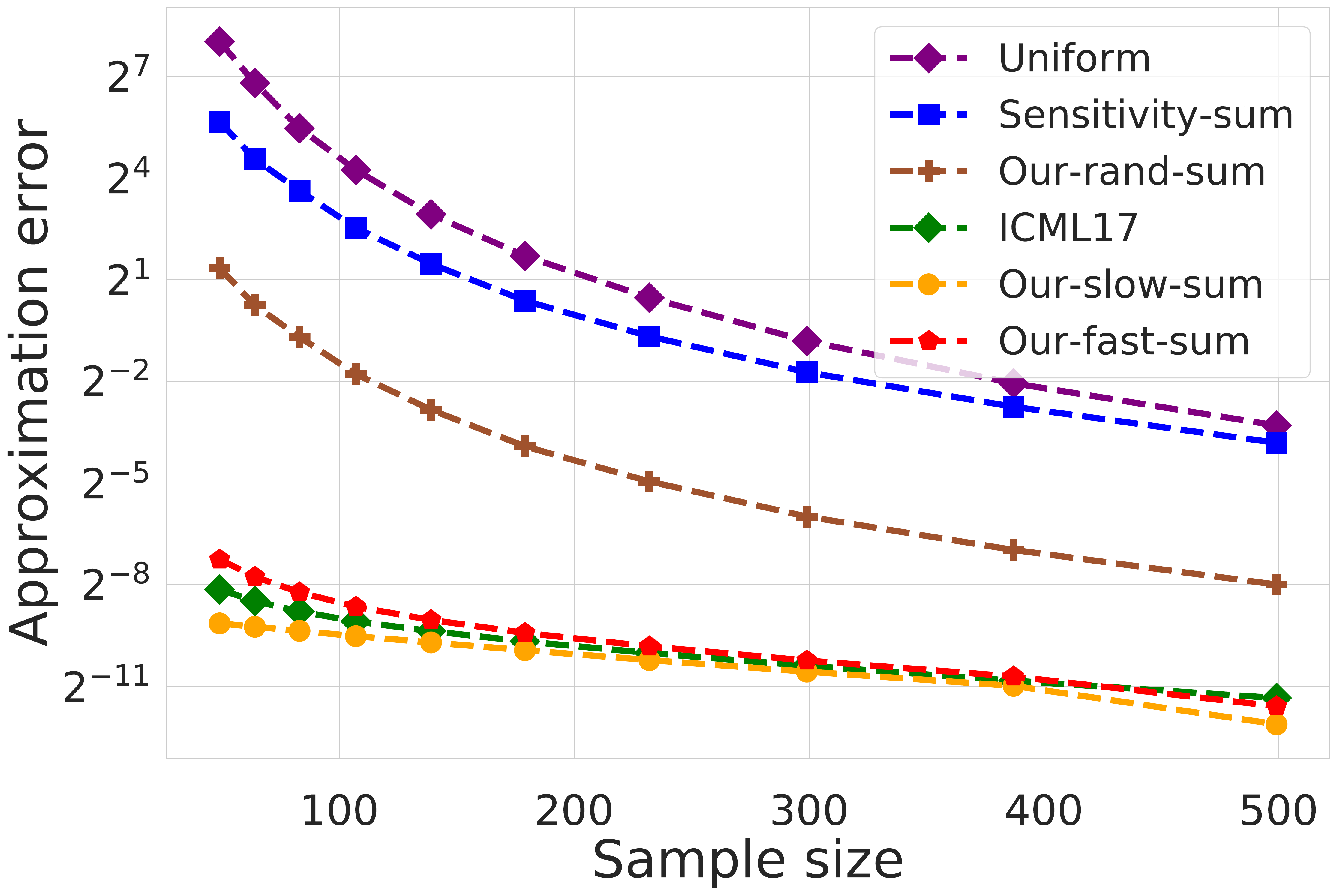}
        \caption{Dataset (i).}
        \label{fig:triperr}
	\end{subfigure}
    \begin{subfigure}[t]{\s\textwidth}
		\centering
		\includegraphics[width = \textwidth]{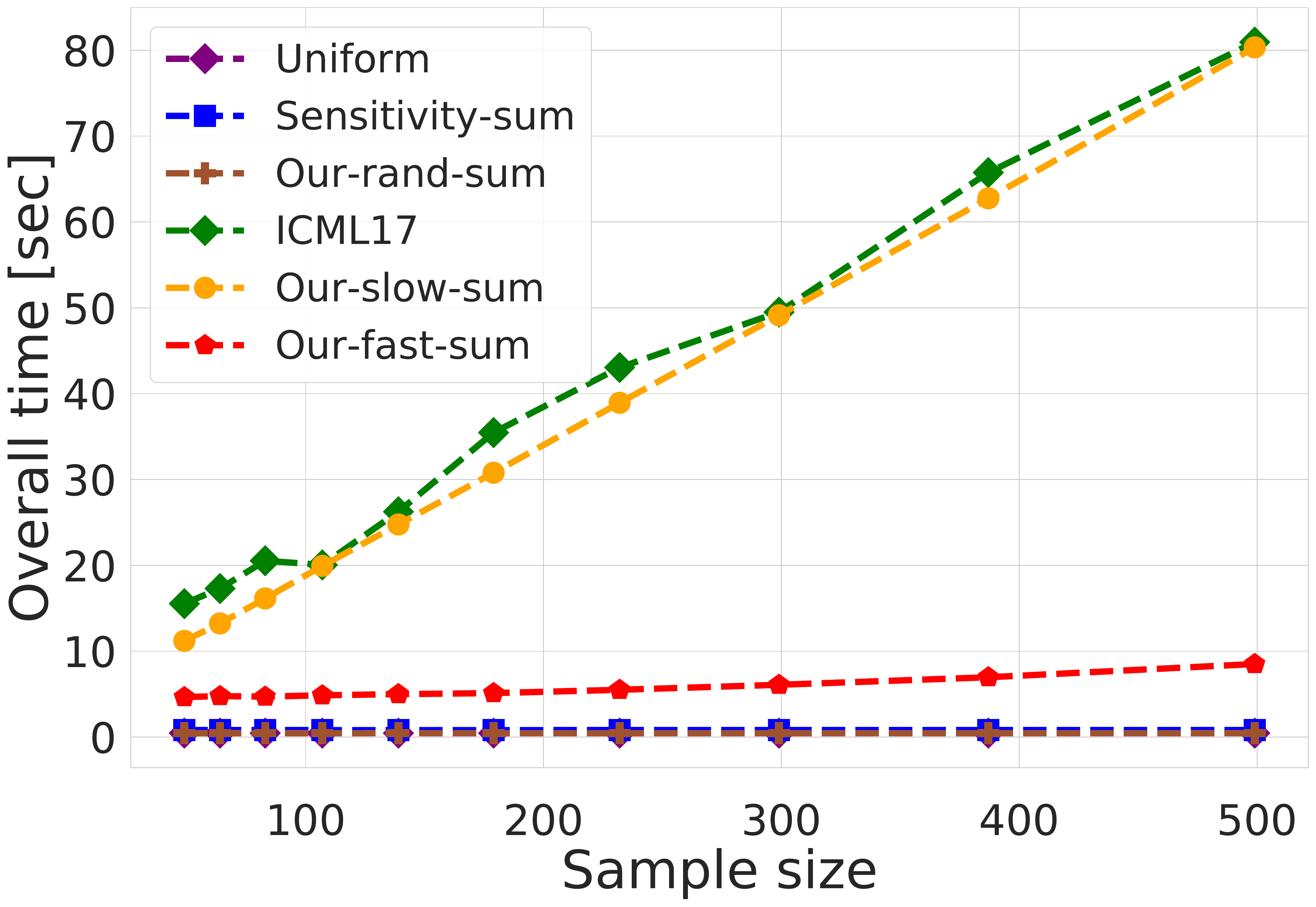}
        \caption{Dataset (i).}
        \label{fig:triptime}
	\end{subfigure}
	  \begin{subfigure}[t]{\s\textwidth}
		\centering		
		\includegraphics[width = \textwidth]{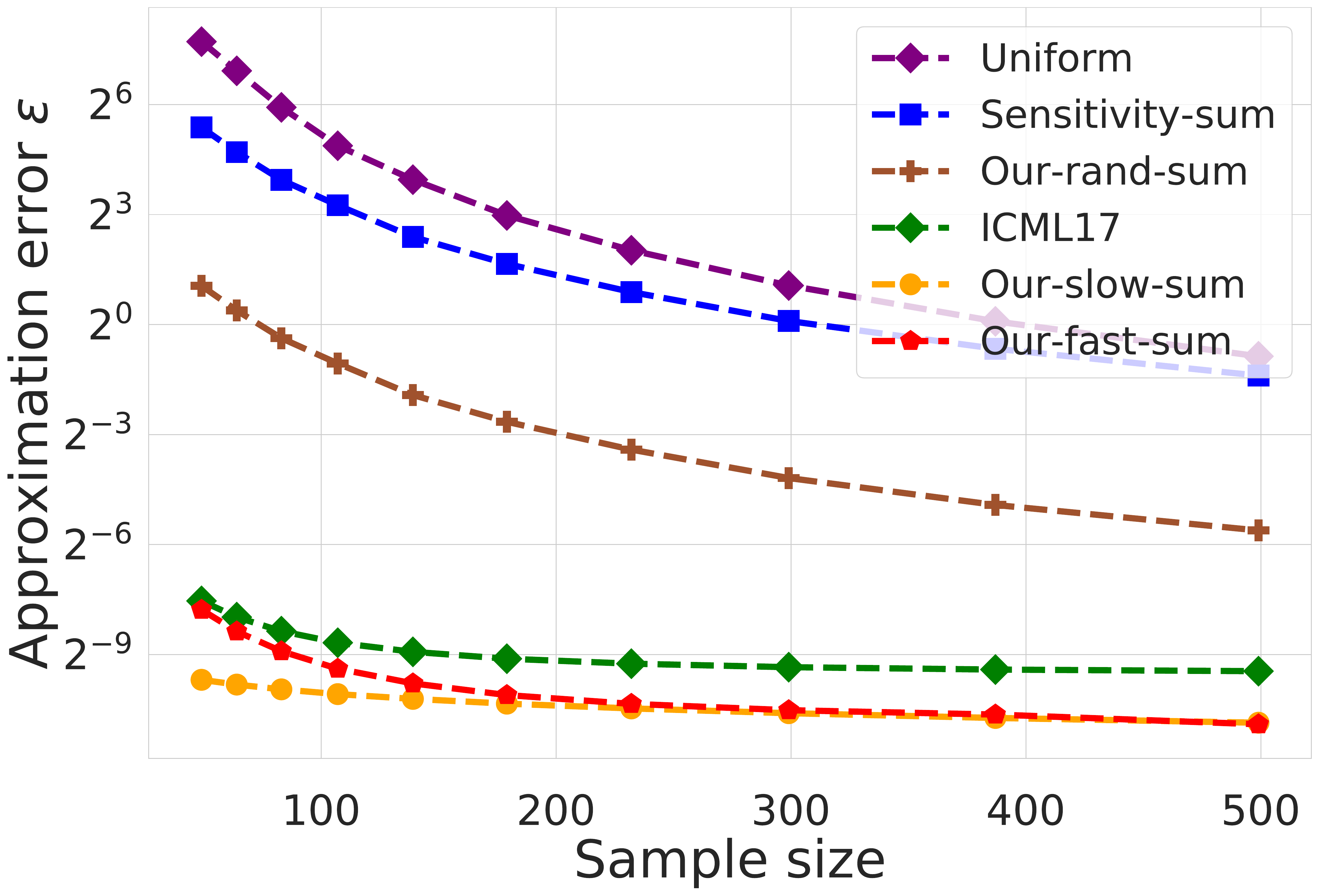}
        \caption{Dataset (ii).}
        \label{fig:USerr}
	\end{subfigure}
    \begin{subfigure}[t]{\s\textwidth}
		\centering		
		\includegraphics[width = \textwidth]{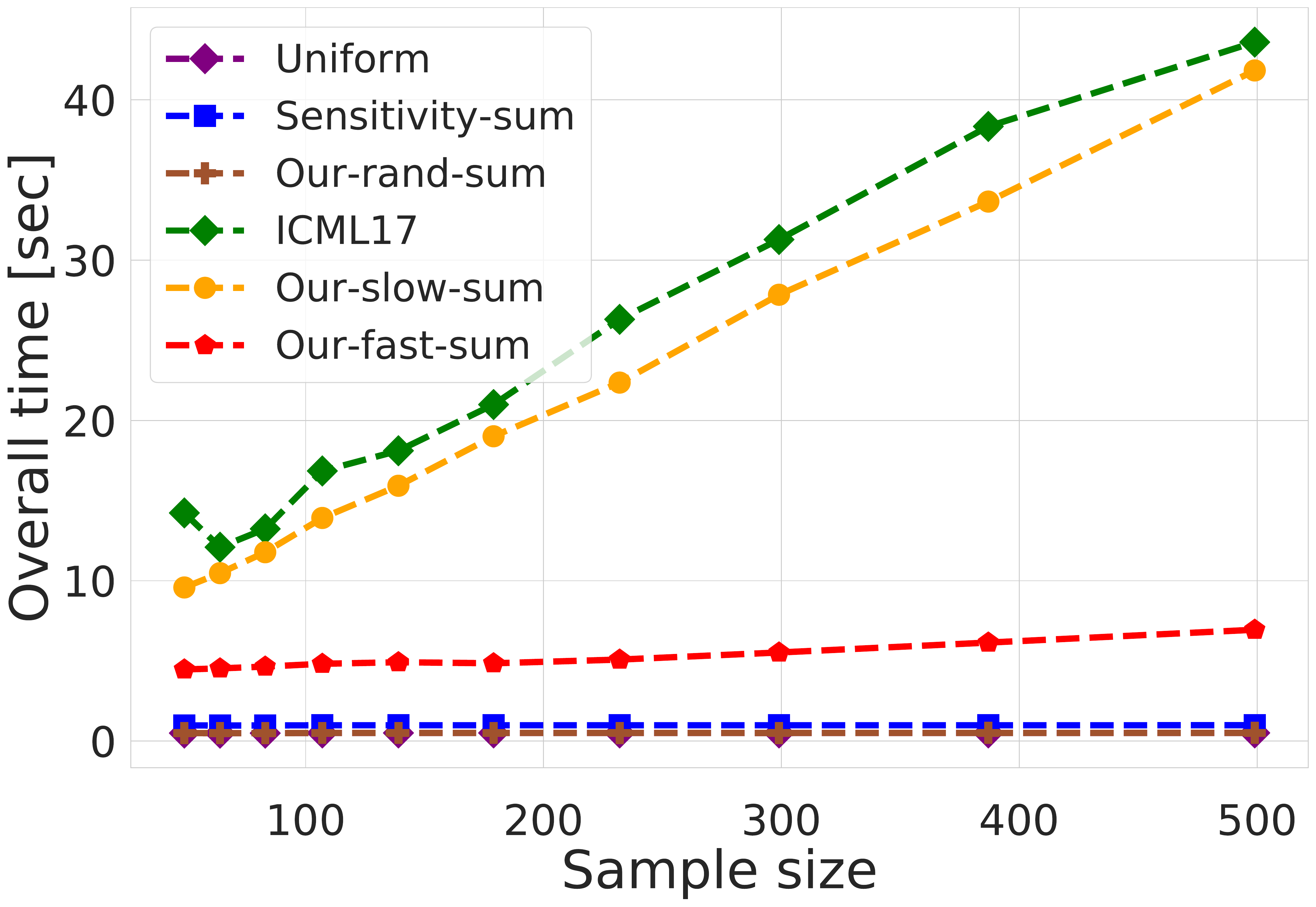}
        \caption{Dataset (ii).}
        \label{fig:UStime}
	\end{subfigure}
    \begin{subfigure}[t]{\s\textwidth}
		\centering
		\includegraphics[width = \textwidth]{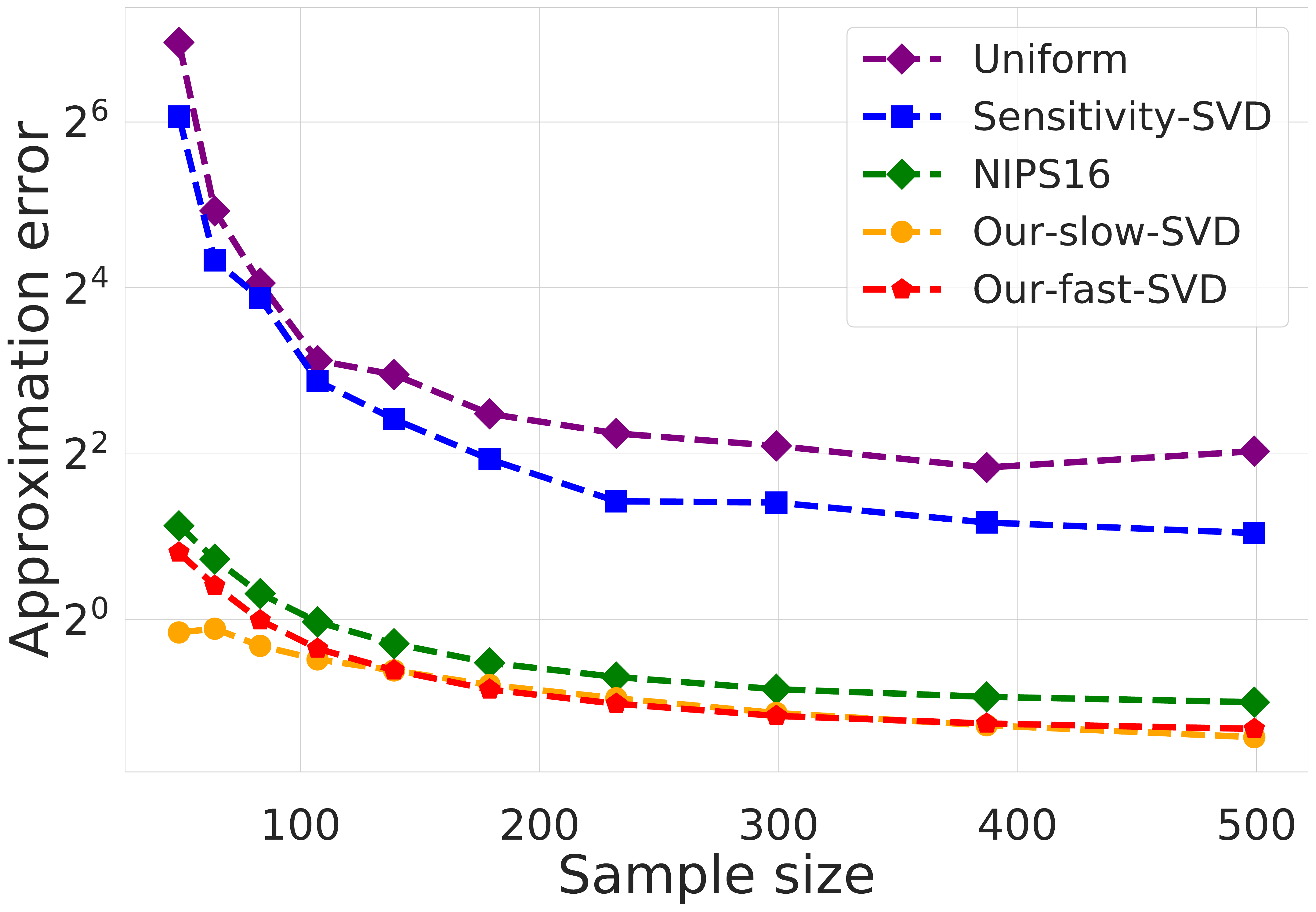}
        \caption{$k=40$, Dataset (iii).}
        \label{fig:twitter40err}
	\end{subfigure}
    \begin{subfigure}[t]{\s\textwidth}
		\centering		
		\includegraphics[width = \textwidth]{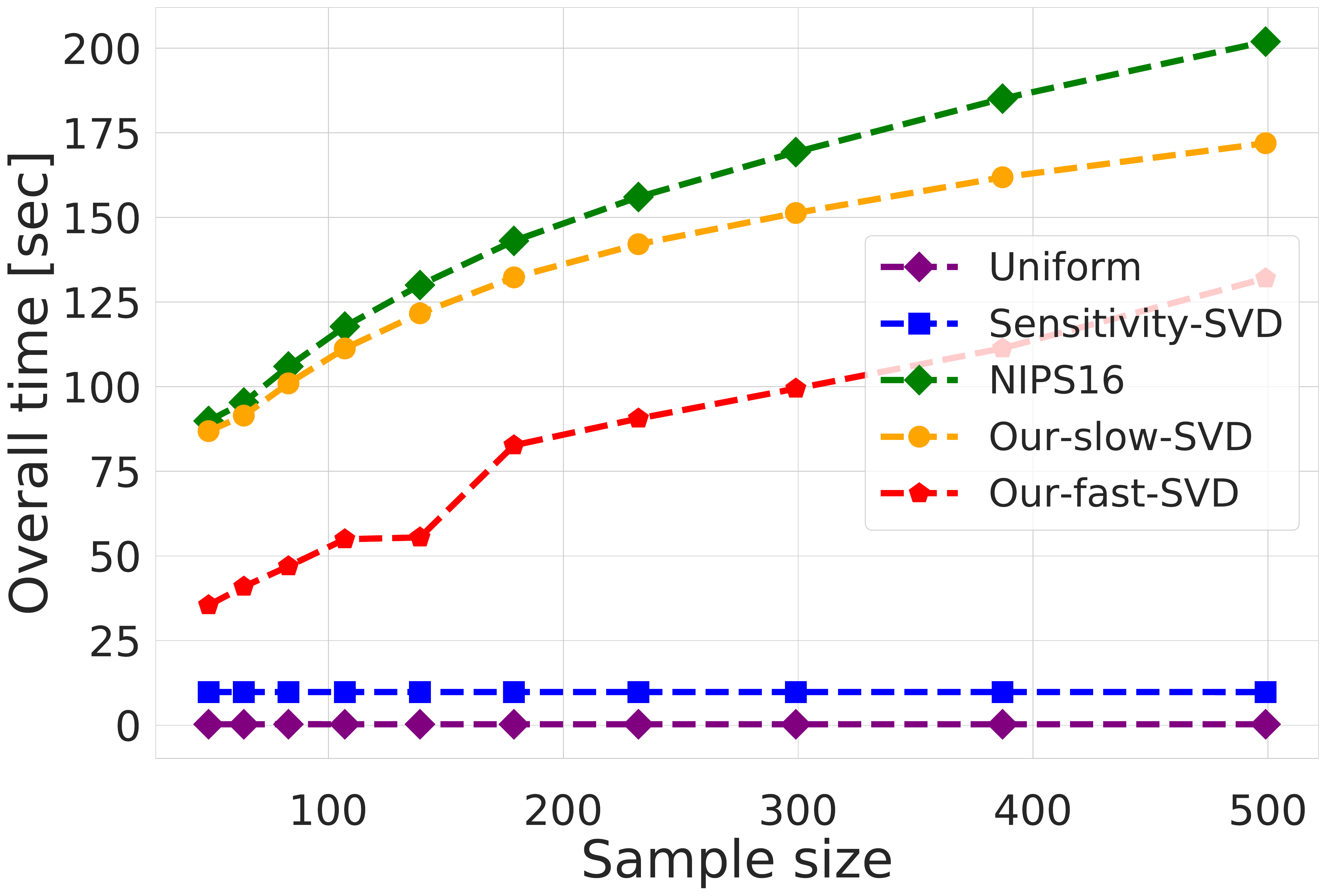}
        \caption{$k=40$, Dataset (iii).}
        \label{fig:twitter40time}
	\end{subfigure}
    \begin{subfigure}[t]{\s\textwidth}
		\centering		
		\includegraphics[width = \textwidth]{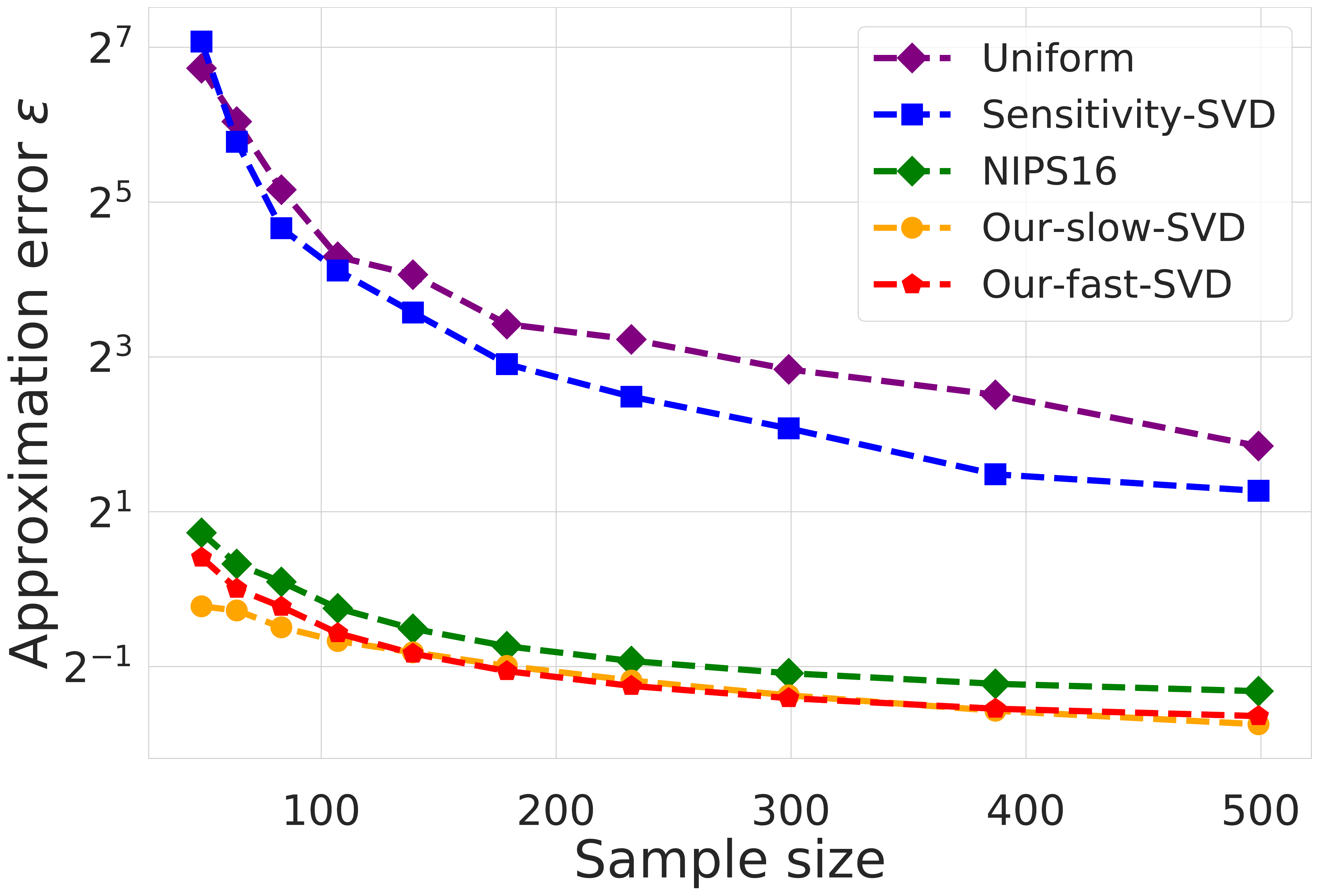}
        \caption{$k=50$, Dataset (iii).}
        \label{fig:twitter50err}
	\end{subfigure}
    \begin{subfigure}[t]{\s\textwidth}
		\centering		
		\includegraphics[width = \textwidth]{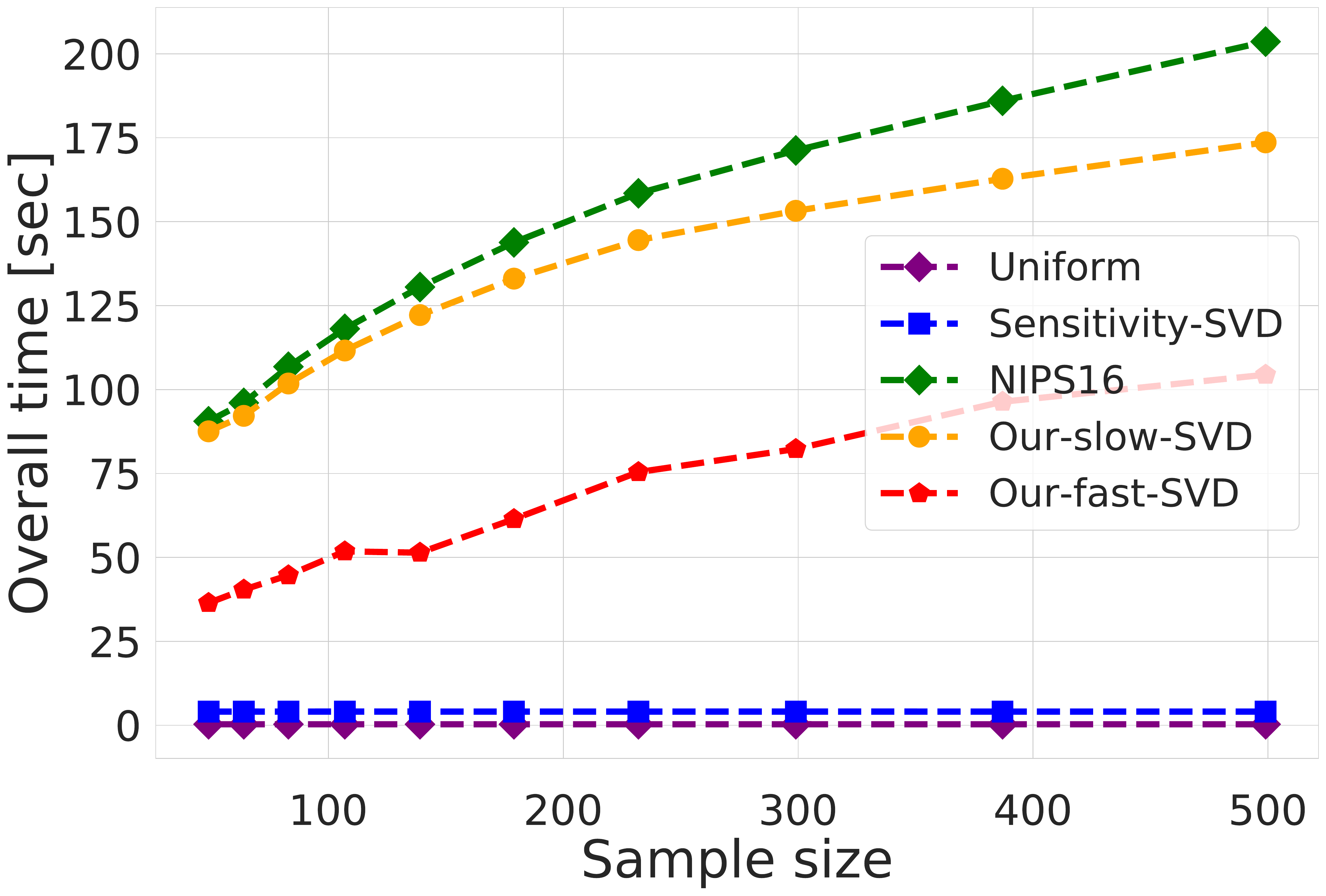}
        \caption{$k=50$, Dataset (iii).}
        \label{fig:twitter50time}
	\end{subfigure}
    \begin{subfigure}[t]{\s\textwidth}
		\centering		
		\includegraphics[width = \textwidth]{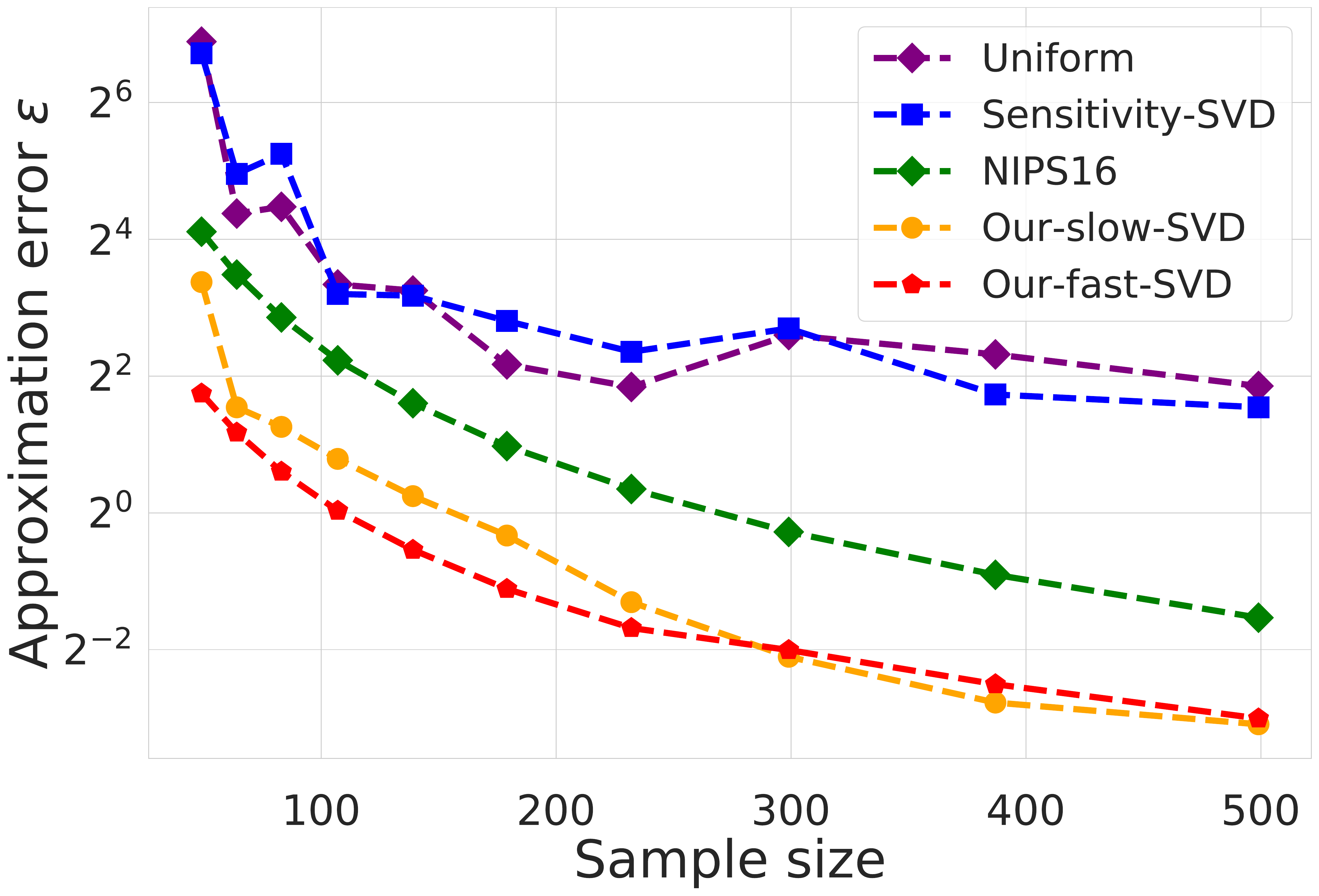}
        \caption{$k=60$, Dataset (iii).}
        \label{fig:twitter60err}
	\end{subfigure}
    \begin{subfigure}[t]{\s\textwidth}
		\centering		
		\includegraphics[width = \textwidth]{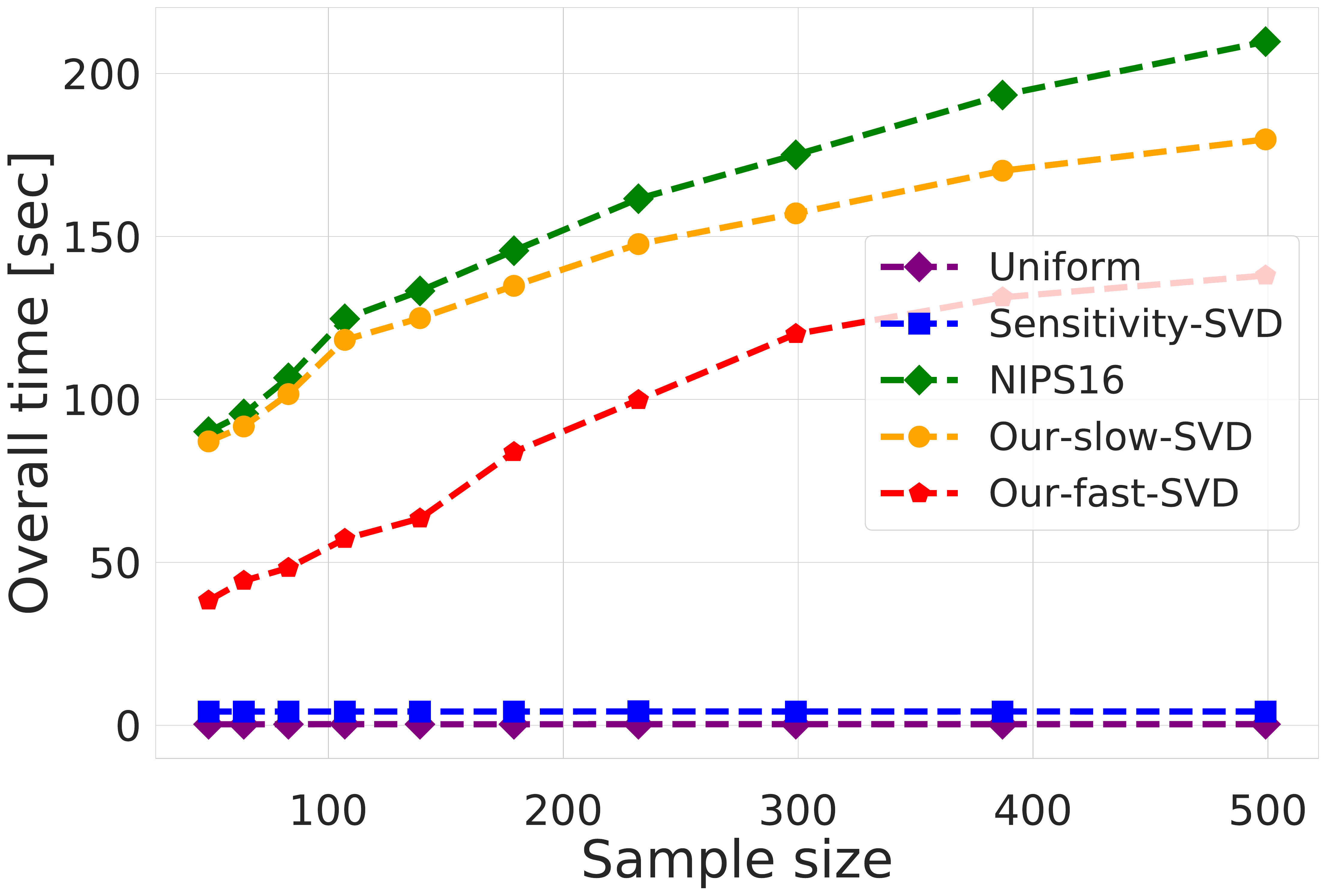}
        \caption{$k=60$, Dataset (iii).}
        \label{fig:twitter60time}
	\end{subfigure}
    \begin{subfigure}[t]{\s\textwidth}
		\centering		
		\includegraphics[width = \textwidth]{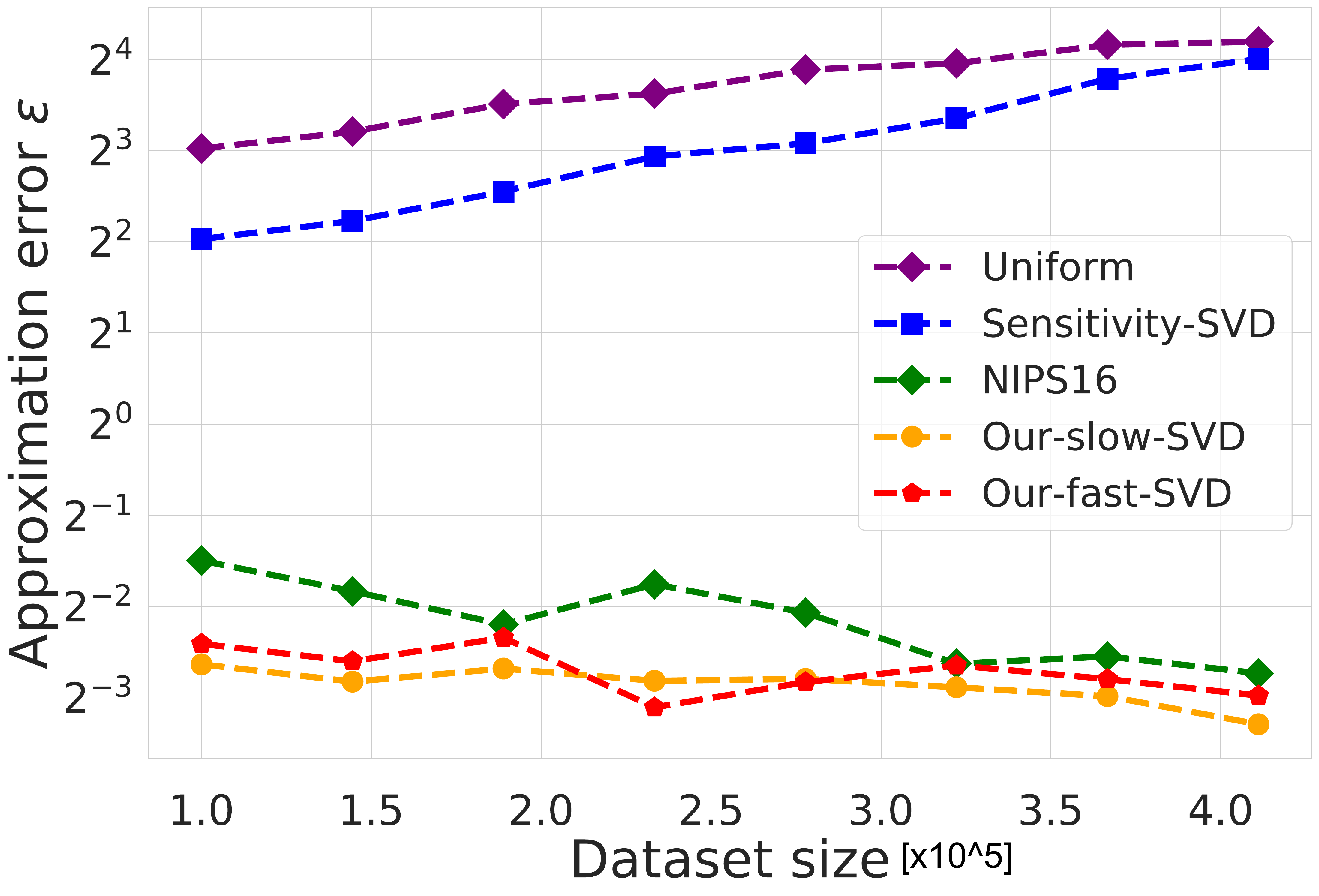}
        \caption{$k=40$, Dataset (iii).}
        \label{fig:TwitterSizeErr}
	\end{subfigure}
    \begin{subfigure}[t]{\s\textwidth}
		\centering		
		\includegraphics[width = \textwidth]{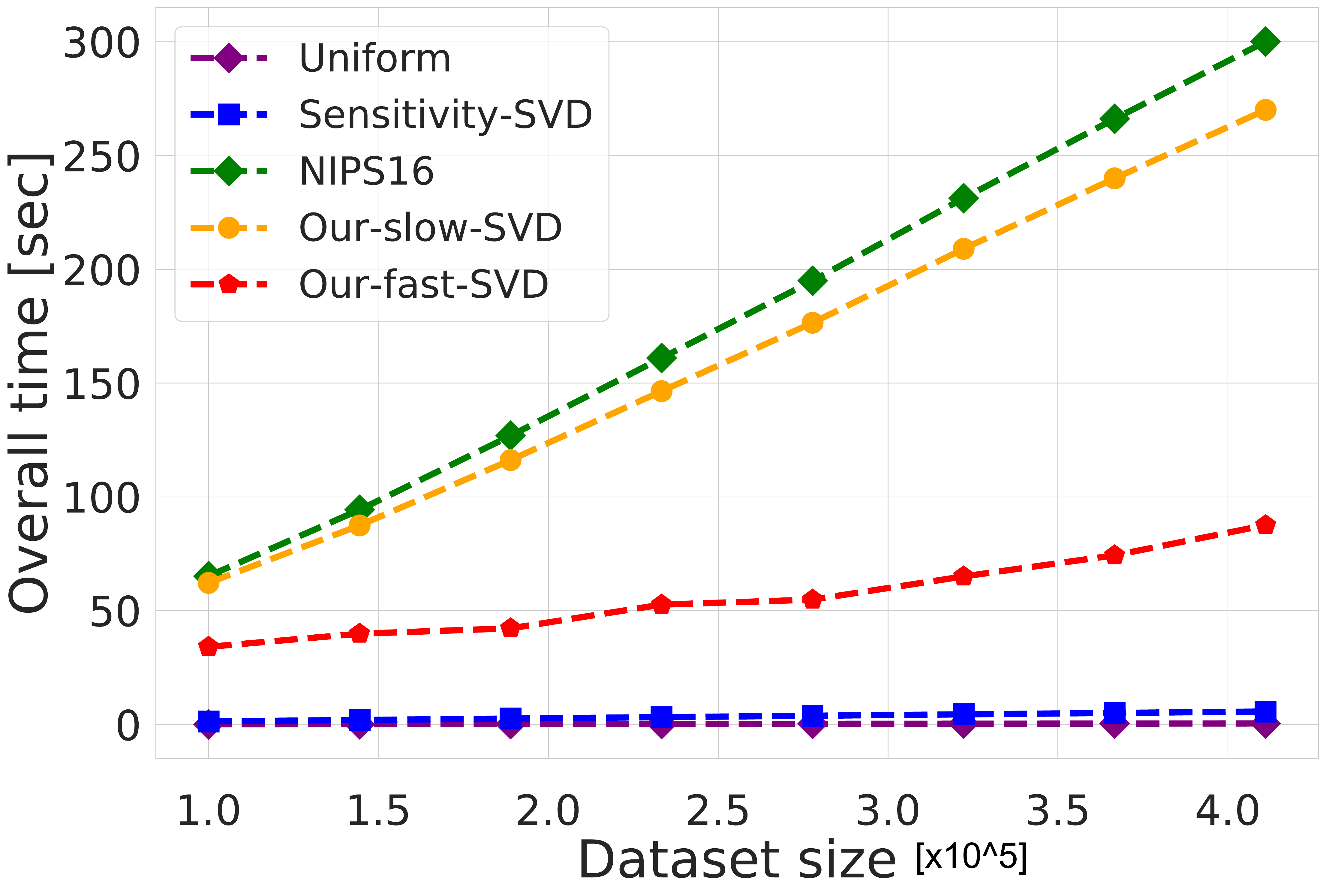}
        \caption{$k=40$, Dataset (iii).}
        \label{fig:TwitterSizeTime}
	\end{subfigure}
    \caption{Experimental results, we used Dataset (iii) in the last $8$ graphs. The $x$ axis in Fig.~\ref{fig:triperr}--\ref{fig:twitter60time} is the size of the subset, while in Fig.~\ref{fig:TwitterSizeErr}--\ref{fig:TwitterSizeTime} is the size of the dataset which we compress to subsample of size $150$.}
\end{figure*}

\section{Conclusions and Future Work} \label{sec:FW}
This paper generalizes the definitions of $\eps$-sample, sensitivities and coreset from the worst case error over every query to smooth (average) $\ell_2$ error. It also suggest deterministic and randomized algorithms for computing them.
Open problems include generalizing these results for other types of norms, or other functions such as M-estimators that are robust to outliers.  We hope that the source code and the promising experimental results would encourage also practitioners to use these new types of approximations.
Normalization via this new sensitivity type reduced the bounds on the number of iterations of the Frank-Wolfe algorithm by orders of magnitude. We believe that it can be used more generally for provably faster convex optimization, independently of coresets or $\eps$-samples. We leave this for future research.


\bibliographystyle{alpha}
\bibliography{references}

\clearpage
\appendix

\section{Competing Algorithms}\label{sec:algcomp}
\textbf{Algorithms. }
Throughout our experiments, we have used the following algorithms:

\textbf{(i) }\uniform: Uniform random sample of the input $Q$, which requires sublinear time to compute.

\textbf{(ii) }\sensonemean: Random sampling based on the ``sensitivity'' for the vector summarization problem~\cite{tremblay2019determinantal}. Sensitivity sampling is a widely known technique~\cite{DBLP:journals/corr/BravermanFL16}, which guarantees that a subsample of sufficient size approximates the input well. Here, the sensitivity of a point $q \in Q$ is $\frac{1}{n} + \frac{\norm{q}^2}{\sum_{q'\in Q}\norm{q'}^2}$. This algorithm requires $O(nd)$ time.

\textbf{(iii) }\icmlfeldman: The coreset construction algorithm from~\cite{feldman2017coresets} (see Algorithm~$2$ there), which runs in $O(nd/\eps)$ time.

\textbf{(iv) }\ourrand: Our coreset construction from Lemma~\ref{prob:onemean}, which requires $O\left(d\log{(\frac{1}{\delta})}^2 + \frac{d\log{(\frac{1}{\delta})}}{\eps}  \right)$ time.

\textbf{(v) }\slowonemean: Our coreset construction from Corollary~\ref{weak-coreset-theorem}, which requires $O(nd/\eps)$ time.

\textbf{(vi) }\fastonemean: Our coreset construction from Corollary~\ref{weak-coreset-cor}, which requires $O(nd + d\log(n)^2/\eps^2)$ time.

\textbf{(vii) }\senssvd: Similar to \sensonemean{} above, however, now the sensitivity is computed by projecting the rows of the input matrix $A$ on the optimal $k$-subspace (or an approximation of it) that minimizes its sum of squared distances to the rows of $A$, and then computing the sensitivity of each row $i$ in the projected matrix $A'$ as $\norm{u_i}^2$, where $u_i$ is the $i$th row the matrix $U$ from the SVD of $A'=UDV^T$; see~\cite{varadarajan2012sensitivity}. This takes $O(ndk)$ time.

\textbf{(viii) }\nipsfeldman: The coreset construction algorithm from~\cite{feldman2016dimensionality} (see Algorithm~$2$ there) which requires $O(nd^2k^2/\eps^2)$ time.

\textbf{(ix) }\slowsvd{}: Corollary~\ref{dim-cor} offers a coreset construction for SVD using Algorithm~\ref{svdccore}, which utilizes Algorithm~\ref{alg:strong}. However, Algorithm~\ref{alg:strong} either utilizes Algorithm~\ref{alg:frankwolf} (see Theorem~\ref{weak-coreset-theorem}) or Algorithm~\ref{algboost} (see Theorem~\ref{weak-coreset-cor}). \slowsvd{} applies the former option, which requires $O(nd^2k^2/\eps^2))$ time.

\textbf{(x) }\fastsvd: Corollary~\ref{dim-cor} offers a coreset construction for SVD using Algorithm~\ref{svdccore}, which utilizes Algorithm~\ref{alg:strong}. However, Algorithm~\ref{alg:strong} either utilizes Algorithm~\ref{alg:frankwolf} (see Theorem~\ref{weak-coreset-theorem}) or Algorithm~\ref{algboost} (see Theorem~\ref{weak-coreset-cor}). \fastsvd{} uses the latter option, which requires $O(nd^2 + d^2\log(n)^2k^4/\eps^4)$ time.

\section{Problem Reduction for Vector Summarization $\eps$-Coresets}

\begin{definition} [Normalized weighted set]\label{defNormalize}
A \emph{normalized weighted set} is a weighted set $(P,w)$ where $P=\br{p_1,\cdots,p_n} \subseteq \REAL^d$ and $w=(w_1,\cdots,w_n)^T\in\REAL^n$ satisfy the following properties:
\begin{enumerate}[label=(\alph*)]
\item Weights sum to one: $\smi w_i = 1$, \label{a}
\item The weighted sum is the origin: $\smi w_ip_i=\mathbf{0}$, and\label{b}
\item Unit variance: $\smi w_i\norm{p_i}^2=1$. \label{c}
\end{enumerate}
\end{definition}

\subsection{Reduction to normalized weighted set}
In this section, we argue that in order to compute a vector summarization $\eps$-coreset for an input weighted set $(Q,m)$, it suffices to compute a vector summarization $\eps$-coreset for its corresponding normalized (and much simpler) weighted set $(P,w)$ as in Definition~\ref{defNormalize}; see Corollary~\ref{cor:reduction}.


\observation \label{to_0_mean} Let $Q=\{q_1,\cdots,q_n\}$ be a set of $n\geq2$ points in $\REAL^d$, $m\in(0,\infty)^n$, $w \in(0,1]^n$ be a distribution vector such that $w=\frac{m}{\norm{m}_1}$, 
 ${\mu=\smi{w_i q_i}}$ and $\sigma=\sqrt{\smi w_i\norm{q_i - \mu}^2 }$.
Let $P=\br{p_1,\cdots,p_n}$ be a set of $n$ points in $\REAL^d$, such that for every $j\in [n]$ we have $p_j=\frac{q_j - \mu}{\sigma}$. Then, $(P,w)$ is the corresponding normalized weighted set of $(Q,m)$, i.e., (i)-(iii) hold as follows:\label{obs1}
\begin{enumerate}[label=(\roman*)]
\item $\smi w_i = 1$, \label{sumw_it}
\item $\smi w_ip_i=\mathbf{0}$, and \label{first_it}
\item $\smi w_i\norm{p_i}^2=1$. \label{sec_it}
\end{enumerate}

\begin{proof}
\begin{align*}
&{\textit{~\ref{sumw_it} }} \smi w_i = 1\text{ immediately holds by the definition of $w$.}\\
&{\textit{~\ref{first_it} }}
\sum_{i=1}^n w_ip_i =  \sum_{i=1}^n w_i \cdot \frac{q_i - \mu} {\sigma} = \frac{1}{\sigma} \left( \sum_{i=1}^n w_iq_i -  \sum_{i=1}^nw_i\mu \right) = \frac{1}{\sigma} \left( \mu -  \sum_{i=1}^nw_i \mu   \right) = \frac{1}{\sigma}  \mu \left(1  -  \sum_{i=1}^nw_i \right) = 0,
\intertext{where the first equality holds by the definition of $p_i$, the third holds by the definition of $\mu$, and the last is since $w$ is a distribution vector.}
&{\textit{~\ref{sec_it} }}
\smi w_i\norm{p_i}^2 = \smi w_i\norm{ \frac{q_i - \mu}{\sigma} }^2=\frac{1}{\sigma^2}\smi w_i\norm{{q_i - \mu}}^2 =\frac{\smi w_i\norm{{q_i - \mu}}^2}{\smi w_i\norm{q_i - \mu}^2} =1,
\end{align*}
where the first and third equality hold by the definition of $p_i$ and $\sigma$, respectively.
\end{proof}

\begin{corollary} \label{cor:reduction}
Let $(Q,m)$ be a weighted set, and let $(P,w)$ be its corresponding normalized weighted set as computed in Observation~\ref{to_0_mean}.
Let $(P,u)$ be a vector summarization $\eps$-coreset for $(P,w)$ and let $u' = \norm{m}_1 \cdot u$. Then $(Q,u')$ is a vector summarization $\eps$-coreset for $(Q,m)$.
\end{corollary}

\begin{proof}
Put $x \in \REAL^d$ and let $y=\frac{x - \mu}{\sigma}$. Now, for every $j\in [n]$, we have that
\begin{align}\label{keyobs}
\norm{q_j - x}^2
=\norm{ \sigma p_j + \mu- (\sigma y + \mu) }^2
= \norm{ \sigma p_j  - \sigma y }^2
= \sigma^2 ||p_j-y||^2,
\end{align}
{where the first equality is by the definition of $y$ and $p_j$.}

Let $(P,u)$ be a vector summarization $\eps$-coreset for $(P,w)$. We prove that $(Q,u')$ is a vector summarization $\eps$-coreset for $(Q,m)$. We observe the following

\begin{equation}\label{reductiontop1}
\begin{split}
\norm{\smi \frac{m_i}{\norm{m}_1} q_i - \smi \frac{u'_i}{\norm{u'}_1} q_i}^2 &= \norm{\smi w_iq_i - \smi \frac{u_i}{\norm{u}_1} q_i}^2 = \norm{\smi (w_i -  \frac{u_i}{\norm{u}_1}) (p_i \sigma + \mu)}^2  \\
& = \norm{\smi (w_i -  \frac{u_i}{\norm{u}_1}) p_i \sigma + \mu \smi (w_i -\frac{u_i}{\norm{u}_1}))}^2  =  \norm{\smi (w_i -  \frac{u_i}{\norm{u}_1}) (p_i \sigma)}^2\\
&\leq \eps \sigma^2
\end{split}
\end{equation}
where the first equality holds since $w=\frac{m}{\norm{m}_1}$ and $\frac{u'}{\norm{u'}_1} = \frac{mu}{m\norm{u}_1} = \frac{u_i}{\norm{u}_1}$, the second holds by~\eqref{keyobs}, and the last inequality holds since $(P,u)$ is a vector summarization $\eps$-coreset for $(P,w)$
\end{proof}

\subsection{Vector Summarization Problem Reduction}
Given a normalized weighted set $(P,w)$ as in Definition~\ref{defNormalize}, in the following lemma we prove that a weighted set $(P,u)$ is a vector summarization $\eps$-coreset for $(P,w)$ if and only if the squared $\ell_2$ norm of the weighted mean of $(P,u)$ is smaller than $\eps$.

\begin{lemma}\label{weak_coreset_lema}
Let $(P,w)$ be a normalized weighted set of $n$ points in $\REAL^d$, $\eps\in (0,1)$, and $u\in \REAL^n$ be a weight vector.
Let $\displaystyle{\overline{p} =\smi \frac{w_i}{\norm{w}_1} p_i = \smi w_i p_i}$, $\displaystyle{\overline{s} = \smi \frac{u_i}{\norm{u}_1}p_i}$, and $\sigma^2 = \norm{p_i -\overline{p} }^2$. Then, $(P,u)$ is a vector summarization $\eps$-coreset for $(P,w)$, i.e.,
\[
\norm{\overline{p} -\overline{s} }^2 \leq \eps \smi w_i\sigma^2
\]
if and only if 	
\[
\norm{\overline{s}}^2  \leq\varepsilon.
\]\label{first}
\end{lemma}

\begin{proof}
The proof holds since $(P,w)$ is a normalized wighted set, i.e., $\overline{p} = 0$, and $\sigma^2 = 1$.
\end{proof}

\section{Frank-Wolfe Theorem} \label{sup:frankwolfpapernotus}


We consider the measure $C_f$ defined in~\cite{clarkson2010coresets}; see equality~$(9)$ in Section~{$2.2$}.
For a simplex $S$ and concave function $f$, the quantity $C_f$ is defined as
\begin{align}
C_f := \sup \frac{1}{\alpha^2}( f(x) + (y-x)^T\nabla f(x)-f(y)),\label{C_F}
\end{align}
where the supremum is over every $x$ and $z$ in $S$, and over every $\alpha$ so that  $y =x + \alpha(z-x)$ is also in $S$.  The set of such $\alpha$ includes $[0, 1]$, but $\alpha$ can also be negative.

\begin{theorem}[\textbf{Theorem~$2.2$ from~\cite{clarkson2010coresets}}] \label{theorem22cla}
For simplex $S$ and concave function $f$, Algorithm~\ref{alg:frankwolf} (Algorithm~{1.1} from~\cite{clarkson2010coresets}) finds a point $x_{(k)}$ on a $k$-dimensional face of $S$ such that
\[
\frac{f(x*) - f(x_{(k)})}{4C_f} \leq \frac{1}{k + 3},
\]
for $k > 0$, where $f(x*)$ is the optimal value of $f$.

\end{theorem}

\section{Proof of Theorem~\ref{thm1}} \label{sup:frankwolf}

\begin{proof}
Let $C_f$ be defined for $f$ and $S$ as in~\eqref{C_F},
and let $f(x^*)$ be the maximum value of $f$ in $S$. Based on Theorem~\ref{theorem22cla} we have:
\begin{enumerate}
\item  $\tilde{u}$ is a point on a $\ceil{\frac{8}{\eps}}$-dimensional face of $S$,  i.e., $\norm{\tilde{u}}_0 \leq  \ceil{\frac{8}{\eps}}$, $u\in S\subset [0,1]^n$ and  $\smi \tilde{u_i}=1$. Hence, \eqref{prop1neded} is satisfied.
    \item  $\frac{ f(x^*) - f(x_{(k)}) }{4C_f} \leq \frac{1}{k+3},$ for every $k\in \br{0,\cdots, \ceil{\frac{8}{\eps}}}.$ \label{result of frank wolf}
\end{enumerate}

Since $f(x) \leq 0$ for every $x\in S$, we have that,
$$f(x^*) = f(w) = -\norm{\smi (w_i-w_i)p_i}^2 =0.$$

Define $A$ to be the matrix of $d\times n$ such that the $i$-th column of $A$ is the $i$-th point in $P$, and let $\mu = \smi w_ip_i$. We get that
\begin{equation}
\begin{split}
f(x) &=-\norm{\smi (w_i-x_i)p_i}^2 = -\norm{\mu -\smi x_ip_i}^2 = -\norm{\mu}^2 +2\mu^T(\smi x_ip_i) -\norm{\smi x_ip_i}^2 \\
&=  -\norm{\mu}^2 +2\mu^TAx - \norm{Ax}^2 =-\norm{\mu}^2 +2x^T  A^T\mu- x^TA^TAx,
\end{split}
\end{equation}
where the second equality holds by the definition of $\mu$, and the fourth equality holds by since $\smi x_i p_i = Ax$ for every $x \in \REAL^n$.

At section~$2.2.$ in~\cite{clarkson2010coresets}, it was shown that for any quadratic function $f':\REAL^n \to \REAL$ that is defined as
\begin{equation}\label{quadraticfunc}
    f'(x)= a+x^Tb + x^TMx,
\end{equation}
where $M$ is a negative semidefinite $n \times n$ matrix, $b\in \REAL^n$ is a vector, and $a\in \REAL$, we have that $C_{f'}\leq diam(A'S)^2$, where $A'\in \REAL^{d\times n}$ is a matrix that satisfies $M=A'^TA'$; see equality~$(12)$ at~\cite{clarkson2010coresets}.

Hence, plugging $a=-\norm{\mu}^2$, $b=2A^T\mu$, and $M=A^TA$ in~\eqref{quadraticfunc} yields that for the function $f$ we have
$C_f\leq diam(AS)^2$, and
\begin{align*}
 diam(AS)^2 =\sup_{a,b \in AS}\norm{a-b}_2^2=\sup_{x,y \in S}\norm{Ax-Ay}^2_2.
\end{align*}
Observe that $x$ and $y$ are distribution vectors, thus
\begin{align*}
\sup_{x,y \in S}\norm{Ax-Ay}_2^2=\sup_{i,j} \norm{p_i-p_j}^2_2.
 \end{align*}
Since $\norm{p_i}\leq1$ for each $i\in [n]$, we have that
\begin{align*}
\sup_{i,j} \norm{p_i-p_j}^2_2 \leq 2.
\end{align*}
By substituting  $C_f\leq 2$, $k=8/{\eps}$, $f(x_{(k)})=f(\tilde{u})=-\norm{\smi(w_i-\tilde{u_i})p_i}^2$,  and $f(x^*)=0$ in~\eqref{result of frank wolf} we get that,
\begin{align}
\frac{ \norm{\smi(w_i-\tilde{u_i})p_i}^2 } {\co} \leq \frac{1}{8/{\eps}+3}.
\end{align}
Multiplying both sides of the inequality by $\co$ and rearranging prove~\eqref{prop2neded} as
\begin{align}
\norm{\smi(w_i-\tilde{u_i})p_i}^2  \leq \frac{\co}{8/{\eps}+3}\leq \frac{\co}{8/{\eps}} =\eps.
\end{align}

\paragraph{Running time:} We have $K=\ceil{\frac{8}{\eps}}$ iterations in Algorithm~\ref{alg:frankwolf}, where each iteration takes $O(nd)$ time, since the gradient of $f$ based on the vector $x=(x_1,\cdots,x_n)^T\in S$ is $-2 A^T \sum\limits_{i=1}^n (w_i - x_i) p_i$. This term is the multiplication between an a matrix in $\REAL^{n \times d}$ and a vector in $\REAL^d$, which takes $O(nd)$ time. Hence, the running time of the Algorithm is $O(\frac{nd}{\eps})$.
\end{proof}

\section{Proof of Theorem~\ref{weak-coreset-theorem}} \label{sup:weakcore}

\begin{proof}
Let $(P,w)$ be the normalized weighted set that is computed at Lines~\ref{wcomp}--\ref{varcomp} of Algorithm~\ref{alg:strong} where $P=\br{p_1,\cdots,p_n}$, and let $\tilde{u} = \frac{u}{\norm{m}_1}$. We show that $(P,\tilde{u})$ is a vector summarization $\eps$-coreset for $(P,w)$, then by Corrolary~\ref{cor:reduction} we get that $(Q,u)$ is a vector summarization $\eps$-coreset for $(Q,m)$.
For every $i\in[n]$ let $w'_i$,$u'_i$,$u_i$ and $p'_i$ be defined as in Algorithm~\ref{alg:strong}, and let $\eps' = \frac{\eps}{16}$. First, by the definition of $u'$ we have that
\begin{align}
\norm{u'}_0 \leq \frac{\co}{\eps'} = \frac{128}{\eps} \label{used1},
\end{align}
and since $u_i = \norm{m}_1 \cdot \frac{2 u_i'}{\norm{(p^T_i\mid 1)}^2}$ for every $i\in [n]$, we get that
\begin{align}
\norm{u}_0 \leq \frac{128}{\eps} \label{used1}.
\end{align}

We also have by Theorem~\ref{thm1} that
\begin{align}
&4\eps' \geq 4\norm{\smi (w_i' -u_i')p'}^2 =4\norm {\smi \frac{ w_i\norm{(p^T_i\mid 1)}^2 - \frac{u_i}{\norm{m}_1}\norm{(p^T_i\mid 1)}^2 } {2}\cdot \frac{(p^T_i\mid 1)^T}{\norm{(p^T_i\mid 1)}^2}}^2\label{by definition of u'} \\
&=\norm { \smi ( w_i  -\tilde{u_i}) \cdot (p^T_i\mid 1)^T}^2 = \norm { \bigg(  \smi (w_i  -\tilde{u_i}) \cdot p^T_i \mid  \smi  ( w_i  -\tilde{u_i} ) \bigg )^T}^2 \label{bforefinish} \\
&\geq \norm {  \smi (w_i  -\tilde{u_i})\cdot p_i }^2, \label{4eps bound}
\end{align}
where the first derivative is by the definition of $u'$ in Algorithm~\ref{alg:strong} at line~\ref{define u'}, the second holds by the definition of $p',w'$ and $u$ at Lines~\ref{p'},~\ref{w'}, and~\ref{u} of the algorithm, the third holds since $\tilde{u}=\frac{u}{\norm{m}_1}$, and the last inequality holds since $\norm{(x\mid y)}^2 \geq x^2$ for every $x \in \REAL^d$ and $y \in \REAL$. Combining the fact that $\smi w_i p_i =0 $ with~\eqref{4eps bound} yields that
\begin{align}
4\eps' \geq \norm {\smi \tilde{u_i}  p_i }^2 .\label{weak-coreset2}
\end{align}

By~\eqref{bforefinish} and since $w$ is a distribution vector we also have that
\begin{align}
4{\eps'} \geq  \bigg| \smi  ( w_i  -\tilde{u_i} ) \bigg|^2 = \abs{1-\smi \tilde{u_i}}^2, \text{which implies that} \quad 2\eps' \geq \abs{1-\smi \tilde{u_i}} \label{bounding 1-sum ui22}
\end{align}
Combining~\eqref{bounding 1-sum ui22} and~\eqref{weak-coreset2} yields that:
\begin{align}
    \norm {\frac{\smi \tilde{u_i}  p_i}{\smi \tilde{u_i}} }^2 \leq \frac{4\eps'}{(1-\eps')^2}\leq 16\eps'=\eps,\label{weak-coreset}
\end{align}
where that second inequality holds since $\eps' = \eps/16 \leq 1/2$.

By Lemma~\ref{weak_coreset_lema},~Corollary~\ref{cor:reduction}, and~\eqref{weak-coreset}, Theorem~\ref{weak-coreset-theorem} holds as
\[
 \norm{\smi \frac{u_i}{\norm{u}_1}q_i - \smi \frac{m_i}{\norm{m}_1} q_i}_2^2 = \norm{\mu_u - \mu_m}_2^2 \leq \eps\sigma^2_m.
\]
\end{proof}
\section{Proof of Theorem~\ref{theorem:fastcoreset}}\label{sup:boost}

\begin{proof}
We use the notation and variable names as defined in Algorithm~\ref{algboost}.

First, we assume that $w(p) >0$ for every $p\in P$, otherwise we remove all the points in $P$ which have zero weight, since they do not contribute to the weighted sum.
Identify the input set $P =\br{p_1,\cdots,p_n}$ and the set $C$ that is computed at Line~\ref{compC} of Algorithm~\ref{algboost} as $C = \br{c_1,\cdots,c_{|C|}}$.
We will first prove that the weighted set $(C,u)$ that is computed in Lines~\ref{compC}--\ref{compW} at an arbitrary iteration satisfies:
\begin{enumerate}[label=(\alph*)]
    \item  $C \subseteq P$, \label{subs1}
    \item $\sum_{p\in C}$ u(p) =1, \label{1.5}
    \item  $\norm{ {\sum_{p\in P}w(p)\cdot p} - {\sum_{p\in C}u(p)\cdot p}}^2 \leq \frac{\eps }{\log{(n)}} $, and \label{22}
    \item $|C| \leq \ceil{\frac{\abs{P}}{2}}.$ \label{33}
\end{enumerate}

Let $(\tilde{\mu},\tilde{u})$ be the vector summarization $\frac{\eps}{\log(n)}$-coreset of $(\br{\mu_1,\cdots,\mu_k},w')$ that is computed during the execution of the current iteration at Line~\ref{compSlowCara}. 
Hence, by Theorem~\ref{thm1}
\begin{equation}\label{eqProps}
 \norm{{\sum_{\mu_i \in \tilde{\mu}} \tilde{u}(\mu_i) \mu_i} - {\sum_{i=1}^{k} w'(\mu_i) \cdot \mu_i}}^2 \leq \frac{\eps} {\log(n)}, \quad \tilde{\mu}\subseteq \br{\mu_1,\cdots,\mu_k} \text{ and } \quad |\tilde{\mu}| \leq \frac{8\cdot \log(n)}{\eps}.
\end{equation}

\paragraph{Proof of~\ref{subs1}. }
Property~\ref{subs1} is satisfied by Line~\ref{compC} as we have that $C \subseteq P$.

\paragraph{Proof of~\ref{1.5}. }
Property~\ref{1.5} is also satisfied since
\begin{equation} \label{sameWeights}
\begin{split}
\sum_{p\in C} u(p) &=
\sum_{\mu_i \in \tilde{\mu}} \sum_{p\in P_i}\frac{\tilde{u}(\mu_i)w(p)}{w'(\mu_i)} =
\sum_{\mu_i \in \tilde{\mu}}\frac{\tilde{u}(\mu_i)}{w'(\mu_i)}  \sum_{p\in P_i} {w(p)} \\
&= \sum_{\mu_i \in \tilde{\mu}}\frac{\tilde{u}(\mu_i)}{\sum_{p\in P_i} w(p)}  \sum_{p\in P_i} {w(p)}=
\sum_{\mu_i \in \tilde{\mu}}\tilde{u}(\mu_i) =1,
\end{split}
\end{equation}
where the first equality holds by the definition of $C$ at Line~\ref{compC} and $w(p)$ for every $p\in C$ at Line~\ref{compW}, and the third equality holds by the definition of $u'(\mu_i)$ for every $\mu_i \in \tilde{\mu}$ as in Line~\ref{compMuiWeight}.

\paragraph{Proof of~\ref{22}. }
By the definition of $w'$ and $\mu_i$, for every $i \in \br{1,\cdots,k}$
\begin{equation}\label{eqDefMui}
\begin{split}
{\sum_{i=1}^{k} w'(\mu_i)\cdot \mu_i} &=
{\sum_{i=1}^{k} w'(\mu_i)\cdot \left(\frac{1}{w'(\mu_i)} \cdot \sum_{p\in P_i}w(p)\cdot p\right)} \\&=
{\sum_{i=1}^{k} \sum_{p\in P_i} w(p)p }=
{\sum_{p\in P} w(p) p}.
\end{split}
\end{equation}

The weighted sum of $(C,u)$ is
\begin{equation} \label{sameWeightedSum}
\begin{split}
{\sum_{p\in C} u(p)p}
= \sum_{\mu_i \in \tilde{\mu}} \sum_{p\in P_i}\frac{\tilde{u}(\mu_i)w(p)}{w'(\mu_i)} \cdot  p
= \sum_{\mu_i \in \tilde{\mu}}\tilde{u}(\mu_i) \sum_{p\in P_i} \frac{w(p)}{w'(\mu_i)} p
= \sum_{\mu_i \in \tilde{\mu}} \tilde{u}(\mu_i) \mu_i,
\end{split}
\end{equation}
where the first equality holds by the definitions of $C$ and $w$, and the third equality holds by the definition of $\mu_i$ at Line~\ref{compMui}.

Plugging~\eqref{eqDefMui} and~\eqref{sameWeightedSum} in~\eqref{eqProps} satisfies~\eqref{22} as

\begin{equation} \label{donly}
\begin{split}
\norm{ {\sum_{p\in P}w(p)\cdot p} - {\sum_{p\in C}u(p)\cdot p}}^2 \leq \frac{\eps }{\log{(n)}}.
\end{split}
\end{equation}

\paragraph{Proof of~\eqref{33}. }
By~\eqref{eqProps} we have that $C$ contains at most $\frac{\log(n)}{\eps}$ clusters from $P$ and at most $|C| \leq \frac{\log(n)}{\eps} \cdot\ceil{\frac{n}{k}} $ points, and by plugging $k=\frac{2\log(n)}{\eps}$ we obtain that $|C| \leq \ceil{\frac{\abs{P}}{2}}$ as required.

We now prove~\ref{smalsubsetcp}--~\ref{runtimeofcorest}.
\paragraph{Proof of~\ref{smalsubsetcp}. }The first condition $\abs{C}\leq \co/\eps$ in~\ref{smalsubsetcp} is satisfied since at each iteration we reduce the data size by a factor of $2$, and we keep reducing until we reach the stopping condition, which is ${O(\frac{\log(n)}{\eps})}$ by Theorem~\ref{thm1} (since we require a $\frac{\eps}{\log(n)}$ error when we use Theorem~\ref{thm1}, i.e., we need coreset of size $O(\frac{\log(n)}{\eps})$).
Then, at Line~\ref{returnInput} when the if condition is satisfied (it should be, as explained) we finally use Theorem~\ref{thm1} again to obtain a coreset of size $\ceil{8/\eps}$ with $\eps$-error on the small data (that was of size $\frac{O(\log(n)}{\eps}$).

The second condition in~\ref{smalsubsetcp} is satisfies since at each iteration we either return such a pair $(C,u)$ at Line~\ref{recursiveCall}, we get by~\ref{1.5} that the sum of weight is always equal to $1$.

\paragraph{Proof of~\ref{epsweakcor}. }
By~\ref{33} we also get that we have at most $\log(n)$ recursive calls. Hence, by induction on~\eqref{2} we conclude that last computed set $(C,u)$ at Line~\ref{recursiveCall} satisfies~\ref{epsweakcor}
\[
\norm{ {\sum_{p\in P}w(p)\cdot p} - {\sum_{p\in C}w(p)\cdot p}}^2 \leq \log(n) \cdot \frac{\eps }{\log{(n)}} =\eps.
\]
At Line we return an $\eps$ coreset for the input weighted set $(P,w)$ that have reached the size of $(\frac{\log(n)}{\eps})$. Hence, the output of a the call satisfies
\[
\norm{ {\sum_{p\in P}w(p)\cdot p} - {\sum_{p\in C}w(p)\cdot p}}^2 \leq 2\eps.
\]

\paragraph{Proof of~\ref{runtimeofcorest}.}
As explained before, there are at most $log(n)$ recursive calls before the stopping condition at Line~\ref{stoprule} is met. At each iteration we  compute the set of means $\tilde{\mu}$, and compute a vector summarization $\left(\frac{\varepsilon}{\log{n}}\right)$-coreset for them. Hence, the time complexity of each iteration is $n'd + \Tfrank(k,d,\frac{\eps}{\log(n)})$ where $n'$ is the number of points in the current iteration, and $\Tfrank(k,d,\frac{\eps}{\log(n)})$ is the running time of Algorithm~\ref{alg:frankwolf} on $k$ points in $\REAL^d$ to obtain a $\frac{\eps}{\log(n)}$-coreset
 . Thus the total running of time the algorithm until the "If" condition at Line~\ref{ifififif} is satisfied is
\[
\sum_{i=1}^{\log(n)} \left(\frac{nd}{2^{i-1}}+ \Tfrank(k,d,\frac{\eps}{\log(n)})\right) \leq 2nd + \log(n) \cdot \Tfrank(k,d,\frac{\eps}{\log(n)}) \in O\left(nd+ \frac{kd}{\frac{\eps}{\log(n)}} \right),
\]
and plugging $k = \frac{2\log(n)}{\eps}$

and observing the the last compression at line~\ref{returnInput} is done on a data of size $O(\frac{\log(n)}{\eps})$ proves~\ref{runtimeofcorest} as the running time of Algorithm~\ref{boosterboost} is
\end{proof}
$$O\left(nd+ {\frac{\log(n)^2d}{\eps^2}} \right).$$

\section{Proof of Corollary~\ref{weak-coreset-cor}}

\begin{proof}
The corollary immediately holds by using Algorithm~\ref{alg:strong} with a small change. We change Line~\ref{define u'} in Algorithm~\ref{alg:strong} to use Algorithm~\ref{algboost} and Theorem~\ref{theorem:fastcoreset}, instead of Algorithm~\ref{alg:frankwolf} and Theorem~\ref{thm1}.
\end{proof}

\section{Proof of Lemma~\ref{prob:onemean}}

We first prove the following lemma:

\begin{lemma}  \label{weak_coreset_markov}
Let $P$ be a set of $n$ points in $\REAL^d$, $\mu = \frac{1}{n}\smip p$, and $\var^2=\frac{1}{n}\smip \norm{p - \mu}^2$.  Let $\eps,\delta \in (0,1)$, and let $S$ be a sample of $m = \frac{1}{\varepsilon\delta}$ points chosen i.i.d uniformly at random from $P$.
Then, with probability at least $1-\delta$ we have that
$$\norm{ \frac{1}{m}\smis p - \mu}^2 \leq  \eps \var^2.$$
\end{lemma}
\begin{proof}
For any random variable $X$, we denote by $E(X)$ and $\text{var}(X)$ the expectation and variance of the random variable $X$ respectively. Let $x_i$ denote the random variable that is the $i$th sample for every $i\in [m]$. Since the samples are drawn i.i.d, we have
\begin{equation}
\text{var}\left(\frac{1}{m}\smis p\right) = \sum_{i=1}^{m} \text{var}\left(\frac{x_i}{m}\right) =    m \cdot \text{var}\left(\frac{x_1}{m}\right) = m\left(\frac{\var^2}{m^2}\right) = \frac{\var^2}{m} =  \eps\delta\var^2.
\label{variance}
\end{equation}
For any random variable $X$ and error parameter $\eps' \in (0,1)$, the generalize Chebyshev's inequality~\cite{chen2007new} reads that
\begin{equation}\label{eq:GenCheby}
\Pr (\norm{X - E(X)} \geq \eps') \leq \frac{\text{var}(X)}{(\eps')^2}.
\end{equation}
Substituting $X = \frac{1}{m}\smis p$, $E(X) = \mu$ and $\eps' = \sqrt{\eps}\sigma$ in~\eqref{eq:GenCheby} yields that
\begin{equation}
\Pr\left(\norm{ \frac{1}{m}\smis p - \mu} \geq \sqrt{\eps} \var  \right) \leq \frac{\text{var}(\frac{1}{m}\smis p)}{\var^2\eps
}.
\label{chebchev}
\end{equation}
Combining~\eqref{variance} with~\eqref{chebchev} proves the lemma as:
\begin{equation}
\Pr\left(\norm{ \frac{1}{m}\smis p - \mu}^2 \geq \eps \var^2  \right) \leq \frac{\eps\delta\var^2}{\var^2\eps} =\delta .
\end{equation}
\end{proof}

Now we prove Lemma~\ref{prob:onemean}

 \begin{proof}
Let $\br{S_1,\cdots,S_k}$ be a set of $k$ i.i.d sampled subsets each of size $\size$ as defined at Line~\ref{line:sapledsets} of Algorithm~\ref{alg:smallProb}, and let $\overline{s}_i$ be the mean of the $i$th subset $S_i$ as define at Line~\ref{line:si}. Let $\displaystyle{\hat{s}:=\argmin_{x\in \REAL^d}\sum_{i=1}^{k}\norm{\overline{s}_i - x}_2}$ be the geometric median of the set of means $\br{\overline{s}_1,\cdots,\overline{s}_{k}}$.

Using Corollary~$4.1.$ from~\cite{minsker2015geometric} we obtain that
$$\Pr \left(\norm{\hat{s} - \mu} \geq 11\sqrt{\frac {\var^2\log({1.4}/{\delta})}{ \myn } }\right )\leq \delta,$$
from the above we have that
\begin{align}
\Pr \left(\norm{\hat{s}- \mu}^2 \geq 121{\frac {\eps\var^2\log({1.4}/{\delta})}{ 4k } }\right )\leq \delta.\label{paperbound}
\end{align}
Note that
 \begin{align}
\Pr \left(\norm{\hat{s}- \mu}^2 \geq 121{\frac {\eps\var^2\log({1.4}/{\delta})}{ 4k } }\right ) &=
\Pr \left(\norm{\hat{s}- \mu}^2 \geq 30.25 \cdot \eps\var^2{\frac {\log({1.4}/{\delta})}{ \myk } }\right )\label{i1} \\
&\geq \Pr \left(\norm{\hat{s}- \mu}^2 \geq 31\cdot \eps\var^2 \right ), \label{use}
\end{align}
where~\eqref{i1} holds by substituting $k=\myk$ as in Line~\ref{line:defk} of Algorithm~\ref{alg:smallProb}, and~\eqref{use} holds since $\frac {\log({1.4}/{\delta})}{ \myk }<1$ for every $\delta \leq 0.9$ as we assumed.
Combining~\eqref{use} with~\eqref{paperbound} yields,
\begin{align}\label{medianisclose}
 \Pr \left(\norm{\hat{s}- \mu}^2 \geq 31 \cdot \eps \var^2\right ) \leq \delta.
\end{align}

For every $i\in [k]$, by substituting $S = S_i$, which is of size $\frac{4}{\eps}$, in Lemma~\ref{weak_coreset_markov}, we obtain that
\[
\Pr(\norm{\overline{s}_i -\mu}^2 \geq \eps \var^2) \leq 1/4.
\]
Hence, with probability at least $1- ({1}/{4})^{k}$ there is at least one set $S_j$ such that
$$\norm{\overline{s}_j -\mu}^2 \leq \eps\var^2.$$
By the following inequalities:
$$ ({1}/{4})^{k} =  ({1}/{4})^{\myk} \leq ({1}/{4})^{\log(1/\delta)} = 4^{\log(\delta)} \leq 2^{\log(\delta)}=\delta$$
we get that with probability at least $1-\delta$ there is a set $S_j$ such that
 \begin{align}\label{closetothemean}
\norm{\overline{s}_j -\mu}^2 \leq \eps\var^2.
 \end{align}
Combining~\eqref{closetothemean} with~\eqref{medianisclose} yields that with probability at least $(1-\delta)^2$ the set $S_j$ satisfies that
\begin{align}\label{medianandmean}
\norm{\overline{s}_j - \hat{s}}^2 \leq 32 \eps \var^2 .
\end{align}

Let $f:\REAL^d\to [0,\infty)$ be a function such that $f(x)=\sum_{i=1}^{k}\norm{\overline{s}_i - x }_2$ for every $x\in \REAL^d$. Therefore, by the definitions of $f$ and $\hat{s}$,  $\displaystyle{\hat{s}:=\argmin_{x\in \REAL^d}\sum_{i=1}^{k}\norm{\overline{s}_i - x}_2=\argmin_{x\in \REAL^d}f(x)}$. Observe that $f$ is a convex function since it is a sum over convex functions.
By the convexity of $f$, we get that for every pair of points $p,q \in P$ it holds that:
\begin{align}\label{Imsmart}
    \text{if } f(q) \leq f(p) \text{ then } \norm{q-\hat{s}} \leq \norm{p-\hat{s}}.
\end{align}
Therefore, by the definition of $i^{*}$ at in Algorithm~\ref{alg:smallProb} we get that
\begin{align}\label{Imsmart2}
    {i^{*}} \in \argmin_{i\in [k]}\norm{\overline{s}_{i} - \hat{s}}.
\end{align}
Now by combining~\eqref{medianandmean} with~\eqref{Imsmart2} we have that:
  \begin{align}\label{medianandclosetoit}
\Pr\left(\norm{\overline{s}_{i^{*}} - \hat{s}}^2 \leq 32 \eps  \var^2  \right) \geq  (1-\delta)^2 .
 \end{align}
Combining~\eqref{medianandclosetoit} with~\eqref{medianisclose} and noticing the following inequality
 $$(1-\delta)^3 = (1-2\delta +\delta^{2})(1-\delta)  \geq  (1-2\delta )(1-\delta) =1 -\delta -2\delta + 2\delta^2 \geq 1-3\delta,$$
 satisfies Lemma~\ref{weak-probabilistic-coreset-theorem} as,
 \[
 \Pr\left( \norm{  \overline{s}_{i^{*}}  -\mu  }^2 \leq  33\eps  \var^2\right) \leq 1-3\delta.
 \]
 \noindent\textbf{Running time.}
 It takes $O\left( \frac{d\log{(\frac{1}{\delta})}}{\eps}  \right)$ to compute the set of means at Line~\ref{line:si}, and $O\left(d\log{(\frac{1}{\delta})}^2  \right)$ time to compute Line~\ref{closest_mean_33} by simple exhaustive search over all the means. Hence, the total running time is $O\left(d\left (\log{(\frac{1}{\delta})}^2 + \frac{\log{(\frac{1}{\delta})}}{\eps} \right) \right)$.
\end{proof}

\section{Proof of Theorem~\ref{strong-coreset-theorem}} \label{sup:strong}

We first show a reduction to a normalized weighted set as follow:
\begin{corollary}\label{red:1mean}
Let $(Q,m)$ be a weighted set, and let $(P,w)$ be its corresponding normalized weighted set as computed in Observation~\ref{to_0_mean}.
Let $(P,u)$ be a $1$-mean $\eps$-coreset for $(P,w)$ and let $u' = \norm{m}_1 \cdot u$. Then $(Q,u')$ is a $1$-mean $\eps$-coreset for $(Q,m)$.
\end{corollary}

\begin{proof}
Let $(P,u)$ be a $1$-mean $\eps$-coreset for $(P,w)$. We prove that $(Q,u')$ is a $1$-mean $\eps$-coreset for $(Q,m)$. Observe that
\begin{align} \label{firstoneneeded}
\abs{\smi (m_i-u'_i)\norm{q_i-x}^2 } =
\abs{\smi (m_i-u'_i) \sigma^2 \norm{p_i-y}^2 } =
\abs{\smi \norm{m}_1   \sigma^2 (w_i-u_i) \norm{p_i-y}^2 },
\end{align}
where the first equality holds by~\eqref{keyobs}, and the second holds by the definition of $w$ and $u'$.

Since $(P,u)$ is a $1$-mean $\eps$-coreset for $(P,w)$
\begin{align} \label{secondoneneeded}
\abs{\smi \norm{m}_1  \sigma^2 (w_i-u_i) \norm{p_i-y}^2}
\leq\eps {\smi \norm{m}_1  \sigma^2 w_i \norm{p_i-y}^2}
= \eps {\smi m_i \norm{q_i-x}^2 },
\end{align}
where the equality holds by~\eqref{keyobs} and since $ w=\frac{m}{\norm{m}_1} $.

The proof concludes by combining~\eqref{firstoneneeded} and~\eqref{secondoneneeded} as
$$\abs{\smi (m_i-u'_i)\norm{q_i-x}^2 }\leq \eps \smi m_i \norm{q_i-x}^2.$$
\newcommand{\mut}{\tilde{\mu}}
\end{proof}

\subsection{$1$-mean problem reduction}
Given a normalized weighted set $(P,w)$ as in Definition~\ref{defNormalize}, in the following lemma we prove that a weighted set $(P,u)$ is a $1$-mean $\eps$-coreset for $(P,w)$ if some three properties related to the mean, variance, and weights of $(P,u)$ hold.
\begin{lemma}\label{strong_coreset_lema}
Let $(P,w)$ be a normalized weighted set  of $n$ points in $\REAL^d$, $\eps \in (0,1)$, and $u\in \REAL^n$ such that,
\begin{enumerate}
  \item $\norm {\smi u_i p_i} \leq  \eps $,\label{1}
  \item $\abs{1- \smi u_i } \leq \eps$, and \label{2}
  \item $\abs{ 1- \smi u_i \cdot \norm {p_i}^2} \leq \eps$. \label{3}
\end{enumerate}				
Then, $(P,u)$ is a $1$-mean $\eps$-coreset for $(P,w)$, i.e., for every $x\in\REAL^d$ we have that	
\begin{align}
\bigg| \smi (w_i - u_i ) \norm{p_i-x}^2  \bigg| \leq
2\eps \smi w_i\norm{p_i-x}^2.  \label{res}
\end{align}
\end{lemma}

\begin{proof}
First we have that,
\begin{align}
\smi w_i \norm{p_i-x}^2 =\smi w_i\norm{p_i}^2  -2x^T\smi w_i p_i + \norm{x}^2\smi w_i = 1+\norm{x}^2, \label{1+x2}
\end{align}
where the last equality holds by the attributes~\eqref{a}--\eqref{c} of the normalized weighted set $(P,w)$. By rearranging the left hand side of~\eqref{res} we get,
\begin{align}
&\left|\smi (w_i-u_i)\norm{p_i - x}^2 \right| = \left |\smi (w_i-u_i) (\norm{p_i}^2 - 2p_i^Tx + \norm{x}^2 )\right| \\
&\leq \left|  \smi (w_i-u_i)\norm{p_i}^2   \right| + \left|   \norm{x}^2 \smi (w_i-u_i)  \right|  + \left|  2x^T\smi (w_i-u_i)p_i \right| \label{triangle inequality}\\
&=\left| 1- \smi u_i\norm{p_i}^2   \right| +\norm{x}^2  \left| 1- \smi u_i \right|  + \left|  2x^T\smi u_ip_i \right| \label{w_i is distribution and rerranging} \\
&\leq \eps +\eps\norm{x}^2 + 2\norm{x}  \norm{ \smi u_ip_i } \label{almost done},
\end{align}
where~\eqref{triangle inequality} holds by the triangle inequality,~\eqref{w_i is distribution and rerranging} holds by attributes~\eqref{a}--\eqref{c}, and~\eqref{almost done} holds by combining assumptions~\eqref{2},~\eqref{3}, and the Cauchy-Schwarz inequality respectively. We also have for every $a,b\geq 0$ that $2ab \leq a^2 + b^2$, hence,
\begin{align}
2ab = 2\sqrt{\eps}a\frac{b}{\sqrt{\eps}} \leq \eps a^2 + \frac{b^2}{\eps}. \label{2ab bound}
\end{align}
By~\eqref{2ab bound} and assumption~\eqref{1} we get that,
\begin{align}
 2\norm{x}  \norm{ \smi u_ip_i } \leq \eps \norm{x}^2 + \frac{\norm{ \smi u_ip_i }^2}{\eps} \leq \eps \norm{x}^2  +  \frac{\eps^2}{\eps} =\eps \norm{x}^2  + \eps. \label{eps(x +4)}
\end{align}
Lemma~\ref{strong_coreset_lema} now holds by plugging~\eqref{eps(x +4)} in~\eqref{almost done} as,
\begin{align}
\bigg|\smi (w_i-u_i)\norm{p_i - x}^2 \bigg|  &\leq \eps +\eps\norm{x}^2 + \eps \norm{x}^2  + \eps = 2\eps +2\eps \norm{x}^2\\
&  = 2\eps(1+ \norm{x}^2) = 2\eps \smi {w_i \norm{p_i-x}^2} \label{eqFinal},
\end{align}
where the last equality holds by~\eqref{1+x2}.

Observe that if assumptions~\eqref{1},~\eqref{2} and~\eqref{3} hold,  then~\eqref{eqFinal} hold. We therefore obtain an $\eps$-coreset.
\end{proof}

To Proof Theorem~\ref{strong-coreset-theorem}, we split it into $2$ claims:
\begin{claim}\label{claim1}
Let $(Q,m)$ be a weighted set of $n$ points in $\REAL^d$, $\eps\in (0,1)$, and let $u$ be the output of a call to $\strongepscoreset(Q,m,(\frac{\eps}{4})^2)$;  see Algorithm~\ref{alg:strong}.
Then $u=(u_1,\cdots,u_n)\in \REAL^n$ is a vector with $\norm{u}_0\leq \frac{128}{\eps^2}$ non-zero entries that is computed in $O(\frac{nd}{\eps^2})$ time, and $(Q,u)$ is a $1$-mean $\eps$-coreset for $(Q,m)$.
\end{claim}

\begin{proof}
Let $(P,w)$ be the normalized weighted set that is computed at Lines~\ref{wcomp}--\ref{varcomp} of Algorithm~\ref{alg:strong} where $P=\br{p_1,\cdots,p_n}$, and let $\tilde{u} = \frac{u}{\norm{m}_1}$. We show that $(P,\tilde{u})$ is a $1$-mean $\eps$-coreset for $(P,w)$, then by Corollary~\ref{red:1mean} we get that $(Q,u)$ is a $1$-mean coreset for $(Q,m)$.

Let $\eps'=\frac{\eps}{4}$, let  $p'_i  := \frac{(p^T_i\mid 1)^T}{\norm{(p^T_i\mid 1)}^2}$ and $w'_i :=\frac{ w_i\norm{(p^T_i\mid 1)}^2}{2}$ for every $i\in [n]$. By the definition of $u'$ at line~\ref{define u'} in Algorithm~\ref{alg:strong}, and since the algorithm gets ${\eps'}^2$ as input, we have that
\begin{align}
\norm{u'}_0\leq \co/{\eps'}^2 =  \frac{128}{\eps^2}, \label{gar1}
\end{align}
and
\begin{align}
\norm{\smi (w'_i - u_i')p'_i}^2\leq {\eps'}^2 .\label{from alg}
\end{align}
For every $i\in[n]$ let $u_i = \norm{m}_1 \cdot \frac{2 u_i'}{\norm{(p^T_i\mid 1)}^2}$ be defined as at Line~\ref{u} of the algorithm. It immediately follows by the definition of $u=(u_1,\cdots,u_n)$ and~\eqref{gar1} that
\begin{align}
\norm{u}_0\leq 128/{\eps'}^2. \label{gar1used}
\end{align}

We now prove that Properties~\eqref{1}--~\eqref{3} in Lemma~\ref{strong_coreset_lema} hold for $(P,\tilde{u})$. We have that
\begin{align}
2{\eps'} &\geq 2\norm{\smi (w'_i - u_i')p'_i} = 2\norm {\smi \frac{ w_i\norm{(p^T_i\mid 1)}^2 -\frac{u_i}{\norm{m}_1}\norm{(p^T_i\mid 1)}^2 } {2}\cdot \frac{(p^T_i\mid 1)^T}{\norm{(p^T_i\mid 1)}^2}} \nonumber \\
&= \norm { \smi ( w_i  -\tilde{u_i}) \cdot (p^T_i\mid 1)^T}
=\norm { \bigg(  \smi (w_i  -\tilde{u_i}) \cdot p^T_i  \mid  \smi  ( w_i  -\tilde{u_i} ) \bigg )^T} \label{before last} \\
&\geq \norm {  \smi (w_i  -\tilde{u_i})\cdot p_i }, \label{2eps bound}
\end{align}
where the first derivation follows from~\eqref{from alg}, the second holds by the definition of $w'_i$,$u'_i$,$u_i$ and $p'_i$ for every $i\in [n]$, the third holds since $\tilde{u} = \frac{u}{\norm{m}_1}$, and the last holds since $\norm{(x\mid y)} \geq \norm{x}$ for every $x,y$ such that $x\in \REAL^d$ and $y \in \REAL$.

 By~\eqref{before last} and since $w$ is a distribution vector we also have that
\begin{align}
2{\eps'} \geq  \bigg| \smi  ( w_i  -\tilde{u_i} ) \bigg| = \abs{1-\smi \tilde{u_i}}. \label{bounding 1-sum ui}
\end{align}
By theorem~\ref{thm1}, we have that $u'$ is a distribution vector, which yields,
\begin{align*}
2 =  2\smi u'_i = \smi \tilde{u_i}\norm{(p^T_i\mid 1)^T}^2  =  \smi \tilde{u_i}\norm{p_i}^2 + \smi \tilde{u_i},
\end{align*}
By the above we get that $2-  \smi \tilde{u_i}= \smi \tilde{u_i}\norm{p_i}^2$. Hence,
\begin{align}
\abs{ \smi (w_i-\tilde{u_i})\norm{p_i}^2 }
= \abs{ \smi w_i \norm{p_i}^2 - (2-  \smi \tilde{u_i})}
= \abs{1- (2- \smi \tilde{u_i})}
= \abs{\smi \tilde{u_i} - 1} \leq 2{\eps'} \label{bound on diff var*n}
\end{align}
where the first equality holds since $\smi \tilde{u_i}\norm{p_i}^2 = 2-  \smi \tilde{u_i}$, the second holds since $w$ is a distribution and the last is by~\eqref{bounding 1-sum ui}. Now by~\eqref{bound on diff var*n},~\eqref{bounding 1-sum ui} and~\eqref{2eps bound} we obtain that $(P,\tilde{u_i})$ satisfies Properties~\eqref{1}--\eqref{3} in Lemma~\ref{strong_coreset_lema}. Hence, by Lemma~\ref{strong_coreset_lema} and Corollary\ref{red:1mean} we get that
\begin{align}
\abs{\smi (w_i-u_i)\norm{p_i - x}^2 } \leq  4{\eps'} \smi {w_i \norm{p_i-x}^2} = \eps \smi {w_i \norm{p_i-x}^2} .
\end{align}

The running time is the running time of Algorithm~\ref{alg:frankwolf} with $\eps^2$ instead of $\eps$, i.e., $O(nd/\eps^2)$.
\end{proof}

Now we proof the fallowing claim other claim:
\begin{claim}\label{claim2}
Let $(Q,m)$ be a weighted set of $n$ points in $\REAL^d$, $\eps\in (0,1)$. Then in $O(nd + d\cdot \frac{\log(n)^2}{\eps^4})$ we can compute a vector $u=(u_1,\cdots,u_n)^T\in \REAL^n$, such that $u$ has $\norm{u}_0\leq \frac{128}{\eps^2}$ non-zero entries, and $(Q,u)$ is a $1$-mean $(2\eps)$-coreset for $(Q,m)$.
\end{claim}

\begin{proof}
The Claim immediately holds by using Algorithm~\ref{alg:strong} with a small change. We change Line~\ref{define u'} in Algorithm~\ref{alg:strong} to use Algorithm~\ref{algboost} and Theorem~\ref{theorem:fastcoreset}, instead of Algorithm~\ref{alg:frankwolf} and Theorem~\ref{thm1}.
\end{proof}

Combining both Claim~\ref{claim1} with Claim~\ref{claim2} proves Theorem~\ref{strong-coreset-theorem}.

\section{Proof of Corollary~\ref{dim-cor}}\label{sup:svdcore}

\begin{proof}
We consider the variables defined in Algorithm~\ref{svdccore}. Let $X \in \REAL^{d\times(d-k)}$ such that $X^TX = I$, and let $A'=[A|(r,\cdots,r)^T]$. Plugging $A=A'$ into Theorem~3 at~\cite{feldman2016dimensionality}

\begin{align}\label{feldmanroof}
 \abs{1- \frac{\norm{WA'X}^2}{\norm{A'X}^2}} \leq 5 \norm{\smi \tilde{v_i} -W^2_{i,i}\tilde{v_i}}.
 \end{align}

We also have by the definition of $W$ and Theorem~\ref{weak-coreset-theorem}
\begin{align}\label{meuseweak}
 \norm{\smi \tilde{v_i} -W^2_{i,i}\tilde{v_i}} \leq (\eps/k) \sqrt{\smi \norm{\tilde{v_i}}^2 }\leq (\eps/k) {\smi \norm{\tilde{v_i}} },
 \end{align}
where the first inequality holds since $W_{i,i}=u^2_i$ for every $i\in [n]$, and the vector $u\in \REAL^{n}$ is a vector summarization $(\eps/5k)^2$-coreset for $(\br{\tilde{v_1},\cdots,\tilde{v_n}},(1,\cdots,1))$.

Finally, at~\cite{feldman2016dimensionality} they show that $(\eps/5k) {\smi \norm{\tilde{v_i}} } \leq \eps$. Hence, combining this fact with~\ref{feldmanroof}, and~\ref{meuseweak} yields
\begin{align}\label{wohoo}
\abs{1- \frac{\norm{WA'X}^2}{\norm{A'X}^2}} \leq \eps.
\end{align}

Finally, the corollary holds by combing Lemma~4.1 at~\cite{maalouf2019tight} with~\eqref{wohoo}
\end{proof}

\end{document}